\pdfoutput=1

\documentclass{article} 
\usepackage{iclr2025_conference,times}

\usepackage{amsmath,amsfonts,bm}

\def\eqref#1{equation~\ref{#1}}

\def\1{\bm{1}}

\DeclareMathAlphabet{\mathsfit}{\encodingdefault}{\sfdefault}{m}{sl}
\SetMathAlphabet{\mathsfit}{bold}{\encodingdefault}{\sfdefault}{bx}{n}

\def\gD{{\mathcal{D}}}

\def\gM{{\mathcal{M}}}

\newcommand{\E}{\mathbb{E}}

\newcommand{\KL}{D_{\mathrm{KL}}}

\newtheorem{theorem}{Theorem}[section]
\newtheorem{lemma}{Lemma}[section]
\newtheorem{corollary}[theorem]{Corollary}

\newtheorem{definition}{Definition}[section]
\newtheorem{assumption}{Assumption}[section]

\newtheorem{proof}{Proof}[section]

\DeclareMathOperator*{\argmin}{arg\,min}

\usepackage{hyperref}
\usepackage{url}

\usepackage{amssymb}
\usepackage{times}
\usepackage{latexsym}
\usepackage{amsmath}
\usepackage{amssymb}
\usepackage{enumitem}
\usepackage{graphicx}
\usepackage{pifont}
\usepackage{color}
\usepackage{titletoc}
\usepackage{colortbl}
\usepackage{wrapfig}  
\usepackage{lipsum} 
\usepackage{multirow}
\usepackage{listings}
\usepackage{array}
\usepackage{booktabs} 
\usepackage{caption}  

\usepackage{inconsolata}
\usepackage{ulem}
\usepackage{array}
\usepackage{amsfonts}
\usepackage[ruled,linesnumbered,noend]{algorithm2e}
\usepackage{algpseudocode}

\begin{document}

\newcommand{\revision}[1]{\textcolor{blue}{#1}}

\newcommand{\ours}{MMed-RAG}

\title{\ours: Versatile Multimodal RAG System for Medical Vision Language Models}

\iclrfinalcopy

\author{Peng Xia$^{1}$, Kangyu Zhu$^{2}$, Haoran Li$^{3}$, Tianze Wang$^{4}$, Weijia Shi$^{5}$, Sheng Wang$^{5}$, \\ \textbf{Linjun Zhang$^{4}$, James Zou$^{6}$, Huaxiu Yao$^1$}\\ $^1$UNC-Chapel Hill, $^2$Brown University, $^3$Carnegie Mellon University, $^4$Rutgers University, \\ $^5$University of Washington, $^6$Stanford University \quad \texttt{\{pxia,huaxiu\}@cs.unc.edu}}

\newcommand{\fix}{\marginpar{FIX}}
\newcommand{\new}{\marginpar{NEW}}

\maketitle

\vspace{-1em}
\begin{abstract}
\vspace{-1em}
Artificial Intelligence (AI) has demonstrated significant potential in healthcare, particularly in disease diagnosis and treatment planning. Recent progress in Medical Large Vision-Language Models (Med-LVLMs) has opened up new possibilities for interactive diagnostic tools. However, these models often suffer from factual hallucination, which can lead to incorrect diagnoses. Fine-tuning and retrieval-augmented generation (RAG) have emerged as methods to address these issues. However, the amount of high-quality data and distribution shifts between training data and deployment data limit the application of fine-tuning methods. Although RAG is lightweight and effective, existing RAG-based approaches are not sufficiently general to different medical domains and can potentially cause misalignment issues, both between modalities and between the model and the ground truth.
In this paper, we propose a versatile multimodal RAG system, \ours, designed to enhance the factuality of Med-LVLMs. Our approach introduces a domain-aware retrieval mechanism, an adaptive retrieved contexts selection, and a provable RAG-based preference fine-tuning strategy. These innovations make the RAG process sufficiently general and reliable, significantly improving alignment when introducing retrieved contexts. Experimental results across five medical datasets (involving radiology, ophthalmology, pathology) on medical VQA and report generation demonstrate that \ours\ can achieve an average improvement of 43.8\% in the factual accuracy of Med-LVLMs. 
Our data and code are available in \href{https://github.com/richard-peng-xia/MMed-RAG}{https://github.com/richard-peng-xia/MMed-RAG}.
\end{abstract}
\vspace{-1em}
\section{Introduction}

Artificial Intelligence (AI) has already transformed healthcare and still has a lot of potential for further advancements~\citep{tuauctan2021artificial,wang2019artificial,wang2025screening,ye2021unified,tu2024towards,xia2024generalizing,hu2024ophnet,ju2024explore}. Recently, Medical Large Vision-Language Models (Med-LVLMs) have shown great promise for advancing interactive and intelligent diagnosis~\citep{li2023llava,moor2023med,zhang2023pmc,wu2023towards}. Despite this potential~\citep{li2023comprehensive,wu2023can,shi2024survey}, current Med-LVLMs still face significant reliability issues, particularly their tendency to generate non-factual medical responses~\citep{xia2024cares,royer2024multimedeval,chen2024detecting,jiang2024medthink}, making them unreliable in critical medical applications. These factuality issues raise serious concerns when deploying such models in clinical settings, where even small diagnostic errors could lead to severe consequences for patient care.

Recently, researchers have begun to focus on improving the factuality of Med-LVLMs through various techniques, including fine-tuning~\citep{li2023llava,moor2023med,thawkar2023xraygpt,zhang2023pmc,chen2024huatuogpt} and retrieval-augmented generation (RAG)~\citep{xia2024rule,he2024meddr,sun2024fact}. Fine-tuning is a direct method to improve model performance, but it faces several limitations in the medical field. First, there is a lack of sufficient high-quality labeled data for fine-tuning in the medical domain. Additionally, a distribution gap often exists between the training data and the real-world deployment data~\citep{schrouff2022diagnosing}, leading to significantly worse model performance during deployment. Hence, RAG has emerged as a viable alternative by providing external references during the inference stage, enhancing the factuality of Med-LVLMs~\citep{wu2023ragtruth,gao2023retrieval}. However, despite its advantages, current RAG implementations in Med-LVLMs have significant limitations. First, these methods tend to be \textit{dataset-specific}, reducing their generalizability across various medical domains. Second, these models are still facing \textit{misalignment issues} that lead to factuality problems. This misalignment may arise from the impact of adding RAG on the original Med-LVLMs' \textit{cross-modality alignment}, as well as on the \textit{overall alignment} between the model and ground truth.

To address these challenges, we propose a versatile factual \textbf{M}ultimodal \textbf{Med}ical \textbf{RAG} system called \textbf{\ours}. Specifically, \ours\ first introduces a domain-aware retrieval mechanism, designed to handle different domains of medical images more effectively. Here, we design a domain identification module to adaptively select a corresponding retrieval model given the input medical image. Secondly, we include a adaptive calibration approach for selecting the number of retrieved contexts. Lastly, \ours\ incorporates RAG-based preference fine-tuning to enhance cross-modality alignment and overall alignment with ground truth. The preference pairs are designed to achieve two goals: first, to improve cross-modality alignment by encouraging the model to avoid generating responses without utilizing input medical images, even the responses are correct; second, to improve overall alignment by encouraging the model to understand retrieved contexts when unsure, while avoiding interference from irrelevant retrieved information.

The primary contribution of this paper is \ours, a versatile multimodal RAG system designed specifically for Med-LVLMs to generate more factual responses. Under mild assumptions, our theoretical analysis demonstrates that \ours\ mitigates both cross-modality misalignment and overall misalignment with ground truth. Furthermore, empirical results on five medical multimodal datasets, covering three medical image modalities (radiology, pathology, and ophthalmology), show that \ours\ significantly improves the factual accuracy of Med-LVLMs, achieving improvements of 18.5\% and 69.1\% on Medical VQA and report generation tasks, respectively, compared to the original Med-LVLM. These empirical findings further demonstrate the effectiveness of our proposed components and support the theoretical analysis in addressing misalignment issues.
\vspace{-0.3em}
\section{Preliminaries}
\vspace{-0.3em}
In this section, we will provide a brief overview of Med-LVLMs and preference optimization.

\noindent
\textbf{Medical Large Vision Language Models}.
Med-LVLMs bridge LLMs with medical visual modules, allowing the model to take medical image \( x_v \) and clinical query \( x_t \) as input \( x \), and autoregressively predict the probability distribution of the next token. The text output is denoted as \( y \).

\noindent
\textbf{Preference Optimization}.
Preference optimization has achieved remarkable results in LLM alignment. Give an input $x$, a language model policy $\pi_\theta$ can produce a conditional distribution $\pi_\theta(y\mid x)$ with $y$ as the output text response. 
The recently popular DPO~\citep{rafailov2023direct} utilizes preference data achieve objective alignment in LLMs. The preference data is defined as \begin{small}$\mathcal{D}=\{x^{(i)}, y_w^{(i)}, y_l^{(i)}\}_{i=1}^N$\end{small}, where \begin{small}$y_w^{(i)}$\end{small} and \begin{small}$y_l^{(i)}$\end{small} represent preferred and dispreferred responses given an input prompt $x$. The probably of obtaining each preference pair is
\begin{small} $
    p(y_w\succ y_l)=\sigma(r(x, y_w)-r(x, y_l)),
$ \end{small}
where $\sigma(\cdot)$ is the sigmoid function. 
In DPO, the optimization can be formulated as classification loss over the preference data as:
{\begin{equation}
\small
\begin{array}{l}
\mathcal{L}_{\textit{DPO}}(\pi_\theta; \pi_{\text{ref}}) = -\mathbb{E}_{(x,y_w,y_l) \sim \mathcal{D}} 
\left[ \log \sigma
\left(
\alpha \log \frac{\pi_\theta(y_w | x)}{\pi_{\text{ref}}(y_w | x)}
- \alpha \log \frac{\pi_\theta(y_l | x)}{\pi_{\text{ref}}(y_l | x)}
\right) \right].
\end{array}
\label{eq:dpo}
\end{equation}}
\vspace{-1em}
where $\pi_\theta$ represents the reference policy, which is the LLM fine-tuned through supervised learning.
\section{\ours: A Versatile Medical RAG System}
\label{sec:3}
In this section, as illustrated in Figure~\ref{fig:method}, we will propose \ours, a versatile RAG system for improving the factuality of Med-LVLMs. Specifically, \ours\ consists of three complementary modules. First, we design a domain-aware retrieval mechanism to select the optimal retriever by feeding each given medical image to the domain identification module. Second, to select an optimal number of retrieved contexts and filter out low-quality information, \ours\ adopts a adaptive method by filtering out low-quality information using the similarity scores during the RAG phase. Lastly, we use a RAG-based preference fine-tuning approach to improve the cross-modality alignment and the overall alignment between groundtruth. We detail these steps as follows:

\begin{figure}[t]
    \centering
    \includegraphics[width=0.95\linewidth]{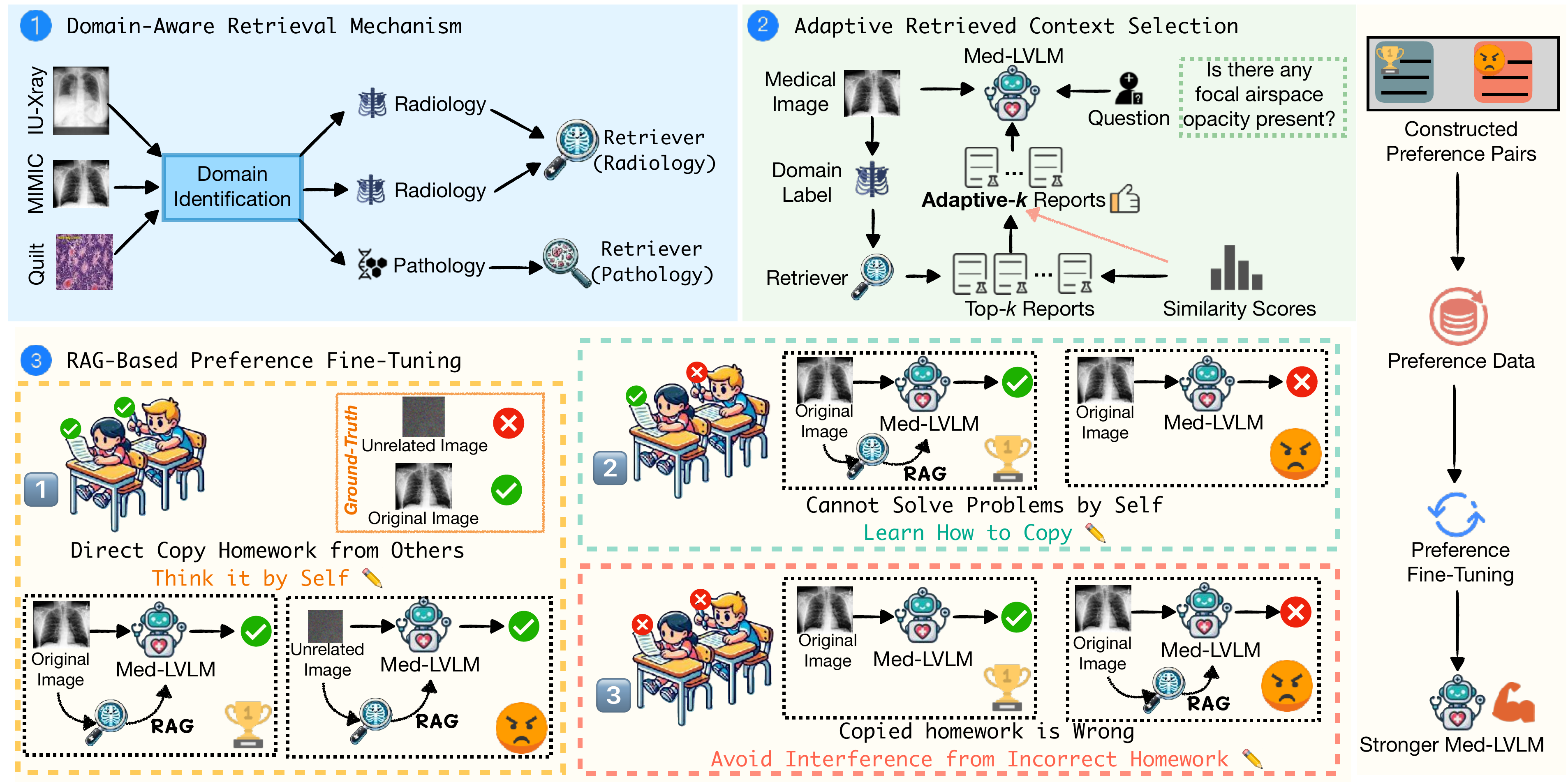}
    \vspace{-0.5em}
    \caption{Overview of \ours, a versatile factual multimodal RAG system designed to enhance the reliability of Med-LVLMs. It introduces a domain-aware retrieval mechanism that effectively handles different domains of medical images by selecting suitable retrieval models. Additionally, it uses an adaptive context selection approach to determine the optimal number of retrieved contexts and employs preference fine-tuning to improve both cross-modality and overall alignment.
    }
    \label{fig:method}
    \vspace{-1em}
\end{figure}

\subsection{Domain-Aware Retrieval Mechanism}
\label{sec:method1}
In \ours, we introduce a domain-aware retrieval mechanism to efficiently handle medical images from different sources (e.g., radiology, pathology, ophthalmology). Specifically, we first employ a domain identification module that assigns a domain label to each input medical image. To achieve this, we create a small dataset with medical images as inputs and their corresponding domain labels as outputs, using this dataset to fine-tune the BiomedCLIP model~\citep{zhang2023biomedclip} to improve its domain awareness. Formally, for a given medical image $x_v$, we predict its domain $d=\mathcal{F}(x_v)$. Based on the assigned domain label $d$, the image $x_v$ is fed into the corresponding multimodal retriever $\mathcal{R}_d(\cdot)$ for knowledge retrieval.

Here, each multimodal retriever $\mathcal{R}_d(\cdot)$ for each domain $d$ is trained through contrastive learning~\citep{radford2021learning}. Specifically, the visual and textual information $X_{img}, X_{txt}$ are processed by their corresponding encoders $\mathcal{E}_{img}(\cdot), \mathcal{E}_{txt}(\cdot)$ to generate textual and visual embeddings $V_{txt}=\mathcal{E}_{txt}(X_{txt}), V_{img}=\mathcal{E}_{img}(X_{img})$. Contrastive learning loss is then applied to maximize the similarity between text and image embeddings representing the same example, while minimizing the similarity between embeddings representing different examples, as defined below:
\begin{equation}
\footnotesize
    \label{eq:clip_loss}
    \begin{aligned}
        \mathcal{L}        & = \frac{\mathcal{L}_{img}+\mathcal{L}_{txt}}{2}, \text{where}\;\;
        \mathcal{L}_{img}=-\frac{1}{N}\sum_{i=1}^{N} \log \frac{\exp(S_{i, i})}{\sum_{j=1}^{N} \exp(S_{i, j})},
        \mathcal{L}_{txt}=-\frac{1}{N}\sum_{i=1}^{N} \log \frac{\exp(S_{i, i})}{\sum_{j=1}^{N} \exp(S_{j, i})},
    \end{aligned}
\end{equation}

where $S \in \mathbb{R}^{N \times N}$ represents the similarity matrix between image and text modalities, calculated as: $S = \frac{V_{img}}{|V_{img}|} \cdot (\frac{V_{txt}}{|V_{txt}|})^T$, where each element $S_{i,j}$ represents the similarity between the image representation of example $i$ and the text representation of example $j$. 

Finally, for the input image $x_t$, after feeding into the corresponding multimodal retriever $\mathcal{R}_d(\cdot)$, the multimodal retriever will retrieves the top-$k$ most similar reports for the image. These retrieved reports $x_r=\mathcal{R}_d(x_v)$ are then provided to the Med-LVLM $\mathcal{M}(\cdot)$ as references to guide the generation.

\subsection{Adaptive Retrieved Context Selection}
\label{sec:method2}
\begin{wrapfigure}{r}{0.36\textwidth}
\vspace{-5em}
\begin{center}
    \includegraphics[width=0.35\textwidth]{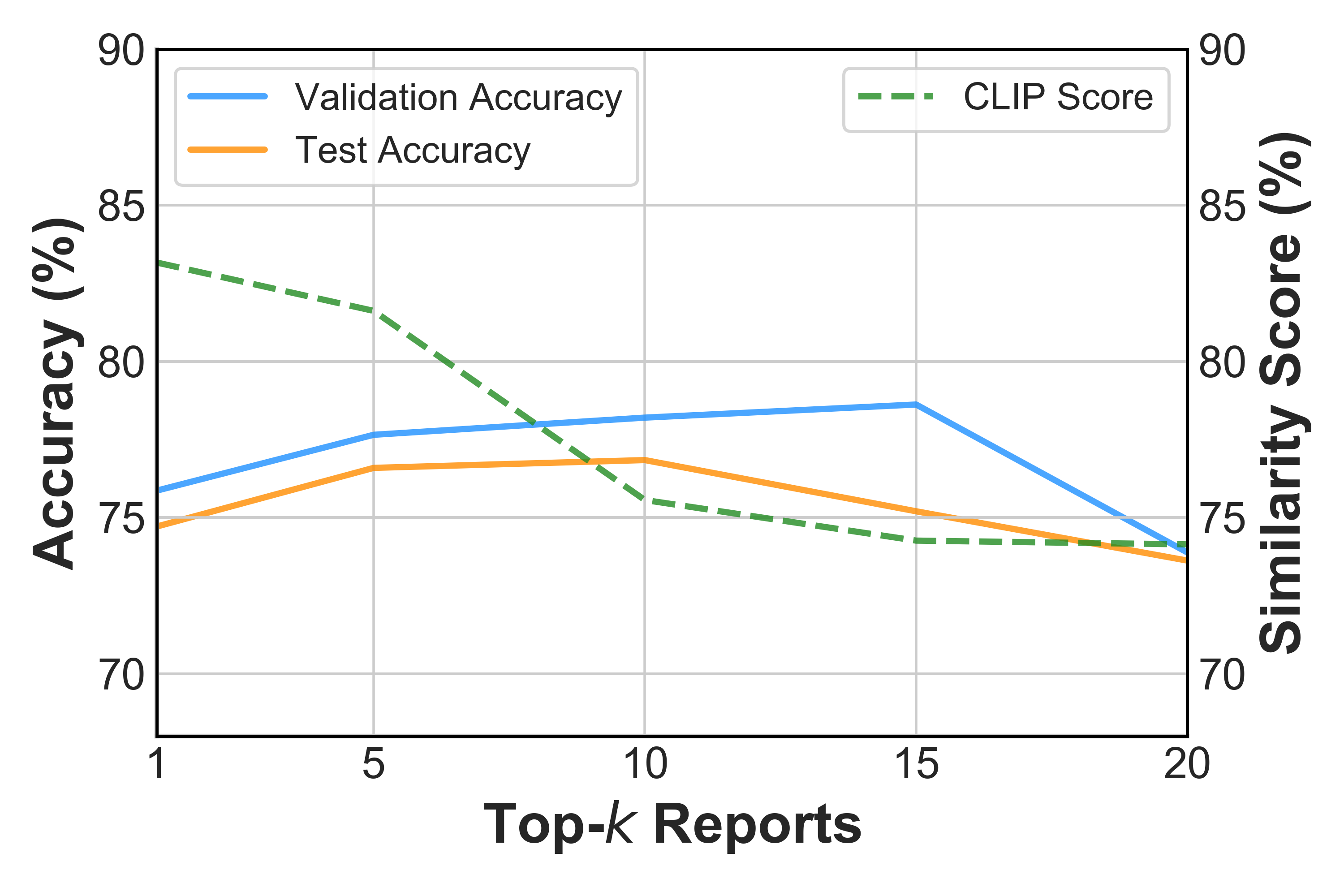}
\end{center}
\vspace{-1em}
\caption{Relations between selected contexts and similarity score.}
\label{fig:mot}
\vspace{-2em}
\end{wrapfigure}
Following the domain-aware retrieval mechanism, the next step is to determine the optimal amount of context to retrieve. Retrieving too much or too little information can result in hallucinations~\citep{xia2024rule}. Current RAG methods applied to Med-LVLMs generally rely on empirical results or fixed values based on validation sets to select the optimal value of the number of retrieved contexts $k$~\citep{xia2024rule,he2024meddr,sun2024fact}. However, the distribution of similarity scores varies depending on the complexity of the image and its alignment with the textual information from the data source. These fixed-$k$ methods do not guarantee optimal performance on target data, as they overlook the similarity scores generated during the retrieval process. To address this, we propose an adaptive method that dynamically selects $k$ based on the similarity scores of the retrieved contexts.
Specifically, during the domain-aware retrieval mechanism phase, the retrieved information is denoted as $x_r(k) = \mathcal{R}_d(x_v; k)$, where $k$ represents the number of retrieved contexts, and the corresponding similarity scores are denoted as $S_k$. For simplicity, when there is no ambiguity, we will refer to $x_r(k)$ as $x_r$.

As illustrated in Figure~\ref{fig:mot}, our method is based on a key observation: the similarity scores (CLIP score in this case) between retrieved contexts often exhibit a sharp decline after a certain number of results (nearly top-9 in this case). This suggests that lower-quality information can still be included among the top-$k$ retrieved contexts when using a fixed-$k$ strategy, especially in cases where the fixed value of $k$ is too large. These lower-quality retrievals introduce noise and irrelevant information, which can significantly impair the model’s ability to generate factual and coherent responses. To mitigate this issue, we draw inspiration from the Gap statistic method used in clustering~\citep{tibshirani2001estimating} and extend this concept to RAG for Med-LVLMs. Specifically, after retrieving the top-$k$ contexts, we perform an additional round of $k$ optimization by analyzing the similarity ratios between consecutive retrievals. These similarity ratios are denoted as $u_i = \log(S_i / S_{i+1})$ for $0 < i \leq k$, where $S_i$ represents the similarity score of the $i$-th retrieved context. When $u_i$ exceeds a predefined threshold $\gamma$, this indicates a substantial drop in relevance, suggesting that the remaining retrievals are less likely to contribute preferredly to the model’s output. At this point $i$, we truncate $k$, effectively discarding the less relevant retrievals that follow. This adaptive truncation mechanism ensures that only the most relevant contexts are retained for generating the final response, reducing the risk of hallucination and improving the factual accuracy of the outputs.

Although the threshold $\gamma$ is fixed, this approach provides a adaptive way to balance the bias and variance in retrieved contexts. By adapting to the characteristics of each input $x_v$, our method enhances the robustness of the retrieval process and ensures that the selection of $k$ is tailored to the specific data at hand, thereby improving overall performance across diverse contexts and tasks.

\subsection{RAG-based Preference Fine-Tuning}
\label{sec:rag-pt}
After context selection, \ours\ supplies Med-LVLM with reliable retrieved information as external knowledge to aid in generating factual responses. However, incorporating this retrieved knowledge may potentially disrupt the original alignment within the existing Med-LVLM, a concern we will elaborate on below:

\noindent \textbf{Alignment Analysis.} In the alignment analysis, we aim to explore how incorporating retrieved context impacts the original alignment in Med-LVLMs, focusing on two key aspects: (1) cross-modality alignment and (2) overall alignment with the ground truth. To evaluate cross-modality alignment, we conduct two tests on LLaVA-Med-1.5~\citep{li2023llava} using the Harvard-FairVLMed~\citep{luo2024fairclip} dataset. First, when replacing the original image with a highly noisy image associated with a different ground truth, the original model gives incorrect answers (the ground truth being the response for the original image). After incorporating RAG, where context is retrieved based on the original image, 55.08\% of these cases return correct answers. This indicates that the model \textit{directly references the retrieved knowledge} without considering the input image, highlighting significant \textit{cross-modal misalignment issues}. Furthermore, 43.31\% of the questions that were originally answered correctly are answered incorrectly after incorporating RAG, suggesting \textit{interference from incorrect retrieval information}, which leads to \textit{overall misalignment with the ground truth}.

\begin{algorithm}[t]
\small
    \caption{Versatile Multimodal RAG System (\ours)}
    \LinesNumbered
    \label{ag:dpo}
    \KwIn{$\mathcal{D}=\{x_v^{(i)},x_t^{(i)},y^{(i)}\}_{i=1}^N$: Dataset; $\pi_\theta$: Parameters of the Med-LVLM; Med-LVLM: $\mathcal{M(\cdot,\cdot)}$; Domain Identification: $\mathcal{F}(\cdot)$; Retriever: $\mathcal{R}(\cdot)$; Noisy Function: $\mathcal{I}(\cdot)$.}
    \KwOut{$\pi_\text{ref}$: Parameters of the reference model.}
    \textbf{$\triangleright$ \textit{Training Stage}} \\
    Initialize $\mathcal{D}_{cm}$ with an empty set\\
    \ForEach{$(x_v,x_t,y) \in \mathcal{D}$}{
        Generate retrieved contexts with an assigned domain label $x_r \leftarrow  \mathcal{R}_{\mathcal{F}(x_v)}(x_v) $ \\
        Generate the noisy image $x_v^* \leftarrow \mathcal{I}(x_v) $ \\
        \textcolor{blue}{$\triangleright$ \textit{Cross-Modality Alignment}} \\
        \If{$\mathcal{M}(x_v,(x_t,x_r))=y$ and $\mathcal{M}(x_v^*,(x_t,x_r))=y$}{
            Select the preferred response $y_{w,o1} \leftarrow y$, dispreferred response $y_{l,o1} \leftarrow \mathcal{M}(x_v^*,(x_t,x_r)) $ \\
            Put $\{(x_v,x_v^*,x_t),y_{w,o1},y_{l,o1}\}$ into $\mathcal{D}_{cm}$ \\
          }
        \textcolor{blue}{$\triangleright$ \textit{Overall Alignment}} \\
        Initialize $\mathcal{D}_{oa}^1$ and $\mathcal{D}_{oa}^2$ with empty set\\
        \If{$\mathcal{M}(x_v,(x_t,x_r))=y$ and $\mathcal{M}(x_v,x_t)\neq y$}{
            Select the preferred response $y_{w,o2} \leftarrow y $, dispreferred response $y_{l,o2} \leftarrow \mathcal{M}(x_v,x_t) $ \\
            Put $\{(x_v,x_t),y_{w,o2},y_{l,o2}\}$ into $\mathcal{D}_{oa}^1$\\
          }
        \If{$\mathcal{M}(x_v,x_t)=y$ and $\mathcal{M}(x_v,(x_t,x_r))\neq y$}{
            Select the preferred response $y_{w,o3} \leftarrow y $, dispreferred response $y_{l,o3} \leftarrow \mathcal{M}(x_v,(x_t,x_r)) $ \\
            Put $\{(x_v,x_t),y_{w,o3},y_{l,o3}\}$ into $\mathcal{D}_{oa}^2$ \\
          }       
    }
    $\mathcal{D}_{pt} = \mathcal{D}_{cm} \cup \mathcal{D}_{oa}$, $\mathcal{D}_{oa} = \mathcal{D}_{oa}^1 \cup \mathcal{D}_{oa}^2$ \\
    \ForEach{$((x_v,x_v^*,x_t),y_{w,o},y_{l,o}) \in \mathcal{D}_{pt}$}{
            Compute the losses $\mathcal{L}_{pt}$ following~\eqref{eq:dpoo} and update $\pi_\text{ref}$
        }
    \textbf{$\triangleright$ \textit{Inference Stage}} \\
    \ForEach{test sample $(x_v,x_t)$}{
        Select top-k retrieved contexts with an assigned domain label $x_r \leftarrow \mathcal{R}_{\mathcal{F}(x_v)}(x_v)$ \\
        Get the predictions of the model w/ RAG-PT $p \leftarrow \mathcal{M}(x_v,(x_t,x_r))$ \\
    }
    
\end{algorithm}

To address cross-modality misalignment and the overall misalignment introduced by incorporating retrieved knowledge, as shown in Algorithm~\ref{ag:dpo}, we propose a RAG-based preference fine-tuning (RAG-PT) approach to fine-tune the target Med-LVLM $\mathcal{M}(\cdot)$. Specifically, RAG-PT constructs two types of preference pairs designed to mitigate both categories of misalignment.

\textbf{Preference Pairs for Cross-Modality Alignment.} We first construct preference pairs aimed at improving cross-modality alignment. In this dataset, we select samples from \begin{small}$\mathcal{D}=\{x_v^{(i)}, x_t^{(i)}, y^{(i)}\}_{i=1}^N$\end{small}, where $x_v$, $x_t$, and $y$ represent the input medical image, clinical query, and ground-truth answer, respectively. For simplicity, we omit the sample index $(i)$ in the following sections. A model's correct response using retrieved knowledge, i.e., $\mathcal{M}(x_v, (x_t, x_r)) = y$, is considered a preferred response $p_i$, where $x_r$ is the retrieved information. A dispreferred response $n_i$ is selected from cases where the model makes a correct inference based on an unrelated image, i.e., $\mathcal{M}(x_v^*, x_t) \neq y$, but $\mathcal{M}(x_v^*, x_t + x_r) = y$, reflecting the model's reliance on the retrieved knowledge. The unrelated images $x_v^*$ are generated through a two-step process: first, we use the retriever to select an image $x_v'$ with the lowest similarity to the target image; then, we introduce diffusion noise into the selected unrelated image. We define the noise step as $s$, and the noised image at step $s$ is expressed as:
\begin{equation}\small
x_v^* = \sqrt{\xi_s} \cdot x_v' + \sqrt{1 - \xi_s} \cdot \epsilon,
\end{equation}
where $\bar{\xi_s} = \prod_{i=0}^{s} \xi_i$ and $ \xi_s \in (0, 1)$ is a hyperparameter. The preference pairs constructed in this stage are denoted as $\mathcal{D}_{cm}$. By comparing the preferred and dispreferred responses in $\mathcal{D}_{cm}$, we encourage the model to prioritize the input medical image when generating responses.

\textbf{Preference Pairs for Overall Alignment.} Second, we construct preference pairs to improve overall alignment, focusing on enhancing the model's ability to effectively leverage retrieved knowledge when generating responses. The preference pairs in this stage are constructed from two subsets. The first subset, $\mathcal{D}_{oa}^1$, is designed to strengthen the model’s comprehension and reasoning abilities regarding the retrieved knowledge. Preferred responses are selected where the model correctly answers based on both the original image and the retrieved information, i.e., $\mathcal{M}(x_v,x_t+x_r) = y$, while dispreferred responses represent cases where the model answers incorrectly based on the image without using retrieval, i.e., $\mathcal{M}(x_v,x_t) \neq y$. Comparing these preferred and dispreferred responses enhances the model’s understanding of the retrieved information and improves the overall effectiveness of RAG. In the second subset, $\mathcal{D}_{oa}^2$, the goal is to mitigate interference from the retrieved knowledge. Preferred responses are selected where the model correctly answers based solely on the original image without using retrieved knowledge, i.e., $\mathcal{M}(x_v,x_t) = y$, while dispreferred responses occur when the model answers incorrectly using both the image and retrieved information, i.e., $\mathcal{M}(x_v,x_t + x_r) \neq y$. This helps the model learn when to rely on its internal knowledge versus retrieved knowledge. Finally, we combine the first and second subsets to form the second set of preference pairs, $\mathcal{D}_{oa} = \mathcal{D}_{oa}^1 \cup \mathcal{D}_{oa}^2$.

Finally, we merge the first and second preference set and denote the preference dataset as \begin{small}$\mathcal{D}_{pt}=\mathcal{D}_{cm} \cup \mathcal{D}_{oa}=\{x^{(i)}, x^{(i)*}, y_{w,o}^{(i)}, y_{l,o}^{(i)}\}_{i=1}^N$\end{small}, where \begin{small}$y_{w,o}^{(i)}$\end{small}, \begin{small}$y_{l,o}^{(i)}$\end{small} are represented as preferred and dispreferred responses, respectively; $x^*$ denotes the noisy data. Based on the curated preferences, we fine-tune Med-LVLM using direct preference optimization~\citep{rafailov2023direct} with the following loss:
{\small \begin{equation}
\begin{array}{l}
\mathcal{L}_{pt} = -\mathbb{E}_{(x,y_{w,o},y_{l,o}) \sim \mathcal{D}} 
\left[ \log \sigma
\left(
\alpha \log \frac{\pi_\theta(y_{w,o} | x)}{\pi_{o}(y_{w,o} | x)}
- \alpha \log \frac{\pi_\theta(y_{l,o} | x^*)}{\pi_{o}(y_{l,o} | x^*)}
\right) \right].
\end{array}
\label{eq:dpoo}
\end{equation}}

    \vspace{-1em}
\section{Theoretical Analysis}
\vspace{-0.5em}
\label{thm anal}
In this section, we provide a theoretical analysis of the model obtained from equation \ref{eq:dpoo} and examine how the image input and retrieved context influences the model. Recall that $x_v, y, x_t,x_r$ denotes input medical image, groundtruth answer, question, and retrieved information, respectively.

\vspace{-0.5em}
\subsection{The Improvement on Cross-Modality Alignment}
We first consider the loss for cross-modality alignment,
{\small
\begin{equation} \small \label{eq: cm loss}
\begin{array}{l}
\mathcal{L}_{cm} = -\mathbb{E}_{(x,y_{w,o},y_{l,o}) \sim \mathcal{D}_{cm}} 
\left[ \log \sigma
\left(
\alpha \log \frac{\pi_\theta(y_{w,o} | x)}{\pi_{o}(y_{w,o} | x)}
- \alpha \log \frac{\pi_\theta(y_{l,o} | x)}{\pi_{o}(y_{l,o} | x)}
\right) \right].
\end{array}
\end{equation}}
where $(x_w,y_{w,o}) \sim q_w(x_w,y_{w,o}|x_t,x_r)$ and $(x_l,y_{l,o}) \sim q_l(x_l,y_{l,o}|x_t,x_r)$ represent distributions of the preferred responses and dispreferred responses on $\mathcal{D}_{cm}$, respectively. Let $x$ denote $(x_v,x_r,x_t)$
\begin{definition} \label{def: wt}
    Define the weight of $x_v$ with respect to $\log\pi_\theta(y|x)$ as
{\small\begin{equation}
    \text{wt}(x_v,\pi_\theta):=\E_{y\sim\pi_\theta(\cdot|x)}\left[
        \frac{\partial}{\partial x_v} \log \pi_\theta(y|x) \right]^2
\end{equation}}
\vspace{-1.5em}
\end{definition}
\noindent Definition \ref{def: wt} describes how $\log\pi_\theta(y|x)$ changes with respect to $x_v$, and the weight is always non-dispreferred. We demonstrate that this is a reasonable definition through Lemma \ref{lem: linear}.
\begin{lemma} \label{lem: linear}
    For linear model $y = \theta_1 x_v + \theta_2 x_t +\epsilon$ such that $\epsilon\sim N (0,1)$, $\text{wt}(x_v,\pi_\theta) = \theta_1^2$
\end{lemma}

\begin{assumption} \label{asump: fun preoperty}
    Let $h(x,y)$, abbreviate as $h$, be
    {\footnotesize \begin{equation}
        h := \left[
        \sum_y\pi_o(y|x)\left(\frac{q_w(y|x)}{q_l(y|x)}\right)^{\frac{1}{\alpha}}\right]^{-1}\left(\frac{q_w(y|x)}{q_l(y|x)}\right)^{\frac{1}{\alpha}}
    \end{equation}}
    
    Assume that $\text{wt}(x_v,\pi_o) < c^2$, where
    {\footnotesize
    \begin{equation}
        c = \sqrt{\left\Vert \sqrt{\pi_o(y|x)}\cdot\frac{\partial}{\partial x_v}h \right\Vert_2^2 + \int \left(\frac{\partial}{\partial x_v}h\right)^2\frac{\pi_o(y|x)}{h} d y} - \left\Vert \sqrt{\pi_o(y|x)}\cdot\frac{\partial}{\partial x_v}h \right\Vert_2
    \end{equation}}
\end{assumption}
\vspace{-1em}
\noindent Assumption \ref{asump: fun preoperty} requires that $x_v$ has a small weight in $\log \pi_o(y|x)$. A model $\pi_o(y|x)$ independent of $x_v$ could satisfy Assumption \ref{asump: fun preoperty}. In this case, the reference model generates answers without using information from the image.
\begin{theorem} \label{thm: increase weight}
Suppose that Assumption \ref{asump: fun preoperty} holds, cross-modality loss increase the weight of $x_v$.
\begin{equation}\small
    \text{wt}(x_v,\pi_\theta) > \text{wt}(x_v,\pi_o)
\end{equation}
\end{theorem}
    \vspace{-1em}
Theorem \ref{thm: increase weight} indicates that when the weight of $x_v$ is too small in the initial model $\pi_o(y|x)$, the cross-modality loss function adjusts the model to place greater emphasis on images, informed by the retrieved data. Intuitively, for any sample $(x,y)$, generating unrelated images causes the policy to rely less on images. By using samples from this distribution as negative samples, the new model diverges from the initial model, increasing its reliance on images.

\subsection{The Improvement on Overall Alignment}
In this section, we analyze the improvement on overall alignment. Let $q_w^1(x_v,y_{w,o}|x_t,x_r)$ and $q_l^1(x_v,y_{l,o}|x_t)$ represent distributions of the preferred responses and dispreferred responses on $\mathcal{D}_{oa}^1$, respectively; $q_w^2(x_v,y_{w,o}|x_t)$ and $q_l^2(x_v,y_{l,o}|x_t,x_r)$ represent distributions of the preferred responses and dispreferred responses on $\mathcal{D}_{oa}^2$, respectively. Overall loss is defined by
\begin{equation} \label{eq: oa loss}
\begin{array}{l}
\mathcal{L}_{oa} = -\mathbb{E}_{(x,y_{w,o},y_{l,o}) \sim \mathcal{D}_{oa}} 
\left[ \log \sigma
\left(
\alpha \log \frac{\pi_\theta(y_{w,o} | x)}{\pi_{o}(y_{w,o} | x)}
- \alpha \log \frac{\pi_\theta(y_{l,o} | x)}{\pi_{o}(y_{l,o} | x)}
\right) \right].
\end{array}
\end{equation}

Consider $\pi$ as the generative distribution underlying $\gM$, construction of $\mathcal{D}_{oa}^1$ and $\mathcal{D}_{oa}^2$ indicate that there is a significant gap between $\pi(y|x_v,x_t,x_r)$ and $\pi(y|x_v,x_t,\tilde{x}_r)$ for $x_r$ generates true answer while $\tilde{x}_r$ generate a false one.
\begin{assumption} \label{asump: lip}
    Assume that $\pi(y|x_x,x_r,x_t): x \to y$ is L-lipschitz continuous on $x_r$ for all $(x_v,x_t,y)$ such that $
    \vert \pi(y|x_v,x_t,x_r) - \pi(y|x_v,x_t,\tilde{x}_r)\vert \leq L \cdot d_x(x_r, \tilde{x}_r)$, where $d_x$ is any distance metric on the text space.
\end{assumption}
Based on Assumption \ref{asump: lip}, $\tilde{x}_r$ can be viewed as being far from the meaningful retrieved information $x_r$, resulting in different weight in the model. Then, we claim in the following theorem that the overall loss in equation \ref{eq: oa loss} can effectively leverage retrieved knowledge while training.

\begin{assumption} \label{asump: fun preoperty 1}
    Let $h_1(x_v,x_t,x_r,y)$, abbreviate as $h_1$, be
    {\small
    \begin{equation}
    h_1:=
    \left[\sum_y\pi_o(y|x)\left(\frac{q_w^1(y|x_v,x_t,x_r)+q_w^2(y|x_v,x_t)}{q_l^1(y|x_v,x_t) + q_l^2(y|x_v,x_t,x_r)}\right)^{\frac{1}{\alpha}}\right]^{-1}\left(\frac{q_w^1(y|x_v,x_t,x_r)+q_w^2(y|x_v,x_t)}{q_l^1(y|x_v,x_t) + q_l^2(y|x_v,x_t,x_r)}\right)^{\frac{1}{\alpha}}
    \end{equation}}
    Assume that $\text{wt}(x_r,\pi_o) < c_1^2$ and $\text{wt}(\Tilde{x}_r,\pi_o) > c_2^2$, where
    {\small\begin{equation}
    \begin{aligned}
        & c_1 = \sqrt{\left\Vert \sqrt{\pi_o}\cdot\frac{\partial h_1}{\partial x_r} \right\Vert_2^2 + \int \left(\frac{\partial h_1}{\partial x_r}\right)^2\frac{\pi_o}{h_1} d y} - \left\Vert \sqrt{\pi_o}\cdot\frac{\partial h_1}{\partial x_r}\right\Vert_2 \\
        & c_2 = \sqrt{\left\Vert \sqrt{\pi_o}\cdot\frac{\partial h_1}{\partial \Tilde{x}_r} \right\Vert_2^2 
        + \int \left( \frac{\partial h_1}{\partial \Tilde{x}_r}\right)^2\frac{\pi_o}{h_1} + \left(\frac{\partial \pi_o}{\partial \Tilde{x}_r}\right)^2\frac{h_1}{\pi_o}
        d y} + \left\Vert \sqrt{\pi_o}\cdot\frac{\partial h_1}{\partial \Tilde{x}_r} \right\Vert_2
    \end{aligned}       
    \end{equation}}
\end{assumption}

\begin{theorem} \label{thm: increase weight x_r}
Suppose that Assumption \ref{asump: fun preoperty 1} holds, then overall loss \ref{eq: oa loss} increase the weight of $x_r$ and decrease the weight of $\Tilde{x}_r$.
\begin{equation}
    \text{wt}(x_r,\pi_\theta) > \text{wt}(x_r,\pi_o),~~~\text{wt}(\Tilde{x}_r,\pi_\theta) < \text{wt}(\Tilde{x}_r,\pi_o)
\end{equation}
\end{theorem}

\noindent Theorem \ref{thm: increase weight x_r} suggests that the model tend to improve the overall alignment. When $\Tilde{x}_r$ generates a false answer, the training procedure tends to reduce the reliance on $\Tilde{x}_r$, resulting in a decrease in the weight assigned to $\Tilde{x}_r$. Conversely, if $x_r$ is helpful for generating the true answer, $\pi_\theta(y|x)$ tend to enhance its use of $x_r$.
    \vspace{-1em}

\section{Experiment}
In this section, we evaluate the performance of \ours, aiming to answer the following questions: (1) Can \ours\ effectively improve the factuality of Med-LVLMs compared to decoding-based and RAG-based baselines? (2) How effective is each proposed component on performance? (3) What is the effect of preference data for different alignment goals? and (4) Does \ours\ actually improve cross-modality alignment and overall alignment?

\subsection{Experimental Setups}
\label{sec:exp}
\textbf{Implementation Details}. 
We use LLaVA-Med-1.5 7B~\citep{li2023llava} as the backbone model. During the preference fine-tuning process, we adapt LoRA fine-tuning~\citep{hu2021lora}. For the training of retriever, the vision encoder is a ResNet-50~\citep{he2016deep}, and the text encoder is a bio-BioClinicalBERT~\citep{alsentzer2019publicly}. We use the AdamW optimizer with a learning rate of $10^{-3}$, weight decay of $10^{-2}$ and a batch size of 32. The model is trained for 360 epochs. For more detailed information on training hyperparameters and training data, please see Appendix~\ref{sec:detail}.

\noindent
\textbf{Baseline Methods}. We compare \ours\ with two types of LVLM hallucination mitigation methods that show promising results in natural image understanding. 1) Decoding-based methods, including Greedy Decoding, Beam Search~\citep{sutskever2014sequence},  DoLa~\citep{chuang2023dola}, OPERA~\citep{huang2023opera}, VCD~\citep{leng2023mitigating}. These methods manipulate the logits of the model's output tokens to enhance factual accuracy. 2) Multimodal RAG-based methods, including MedDr~\citep{he2024meddr}, FactMM-RAG~\citep{sun2024fact}, RULE~\citep{xia2024rule}. Furthermore, we compare the performance with other open-source Med-LVLMs, including Med-Flamingo~\citep{moor2023med}, MedVInT~\citep{zhang2023pmc}, RadFM~\citep{wu2023towards}. 

\noindent
\textbf{Evaluation Datasets}.
\label{sec:dataset}
We utilize five medical vision-language datasets for medical VQA and report generation tasks, i.e., MIMIC-CXR~\citep{johnson2019mimic}, IU-Xray~\citep{demner2016preparing}, Harvard-FairVLMed~\citep{luo2024fairclip}, PMC-OA~\citep{lin2023pmc} (we only select the pathology part) and Quilt-1M~\citep{ikezogwo2024quilt}. These datasets cover radiology, ophthalmology, and pathology. To construct the VQA benchmarks, following~\citep{xia2024cares}, we generate question-answer pairs from medical reports using GPT-4~\citep{openai2023gpt4}, with answers formatted as \textit{yes} or \textit{no}. Pathology images are excluded from the report generation task due to their brief and insufficient descriptions. The detailed dataset descriptions are provided in the Appendix~\ref{sec:app_data}. 

\noindent
\textbf{Evaluation Metrics}.
Following~\citep{jing2017automatic,lin2023medical}, we use Accuracy, F1 Score and AUROC for evaluating medical VQA task, and BLEU Score~\citep{papineni2002bleu}, ROUGE-L~\citep{lin2004rouge} and METEOR~\citep{banerjee2005meteor} for evaluating report generation task.
\vspace{-1em}
\subsection{Main Results}
In this section, we provide a comprehensive comparison with various baseline methods and other open-source Med-LVLMs on medical VQA and report generation tasks. 
\begin{table}[t]
    \centering
    \footnotesize
    \caption{Model performance (\%) of different methods based on LLaVA-Med-1.5 on medical VQA task. Notably, we report the accuracy, F1 score and AUROC. The best results and second best results are highlighted in \colorbox{red!25}{red} and \colorbox{blue!15}{blue}, respectively. }
    \vspace{-1em}
    \resizebox{\linewidth}{!}{
    \begin{tabular}{l|cccccc|ccc|cccccc}
    \toprule
        \multirow{2}{*}{Models} & \multicolumn{6}{c|}{\textbf{Radiology}} & \multicolumn{3}{c|}{\textbf{Ophthalmology}} & \multicolumn{6}{c}{\textbf{Pathology}} \\ \cmidrule(r){2-7} \cmidrule(r){8-10} \cmidrule(r){11-16}     
        & \multicolumn{3}{c}{IU-Xray} & \multicolumn{3}{c}{MIMIC-CXR} & \multicolumn{3}{|c}{Harvard-FairVLMed} & \multicolumn{3}{|c}{Quilt-1M} & \multicolumn{3}{c}{PMC-OA (Pathology)} \\ \cmidrule(r){2-4} \cmidrule(r){5-7} \cmidrule(r){8-10} \cmidrule(r){11-13} \cmidrule(r){14-16}
        & Acc & F1 & AUC & Acc & F1 & AUC & Acc & F1 & AUC & Acc &F1 &AUC & Acc &F1& AUC \\
        \midrule
        LLaVA-Med-1.5 & 75.47 & 64.04 & 67.46 & 75.79 & 80.49 & 68.84 & 63.03&74.11&63.05 & 62.80&72.90&60.03 & 59.28&71.98&54.19  \\ 
        \midrule
        + Greedy & 76.88 & 65.59 & 68.74 & 78.32 & 86.75 & 71.13 & 82.54&85.98&70.09 & 64.72&70.12&58.75 & 58.61&70.42&53.10 \\
        + Beam Search & 76.91 & 66.06 & 68.77 & 81.56 & 86.36 & 73.79 & 80.93&88.08&68.94 & 63.52&69.33&57.65 & 56.29&69.84&52.89 \\
        + DoLa & 78.00 & 66.75 & 72.19 & 81.35  & 85.73 & 72.73 & 76.87&85.53&67.10 & 63.47&69.10&57.58 & 57.71&70.27&52.95 \\
        + OPERA & 70.59 & 61.54 & 63.22 & 69.34 & 76.66 & 62.46 & 71.41&81.37&65.59 & 60.51&66.32&54.79 & 55.32&68.30&51.86 \\
        + VCD & 68.99 & 54.35 & 61.08 & 70.89 & 75.57 & 64.61 & 65.88&77.20&64.16 & 61.43&67.39&55.72 & 55.10&67.94&51.62 \\ \midrule
        + MedDr & 83.33 & 67.80 & 77.15 & 55.16 & 56.18 & 58.47 & 70.17&80.72&64.15 & 68.15&73.23&67.01 & 59.97&69.19&57.01 \\
        + FactMM-RAG & 84.51 & 68.51 & 77.07 & 77.58 & 81.86 & 70.09 & 83.67 & 87.21 & 72.20 & \cellcolor{blue!15}{69.25} & 73.62 & \cellcolor{blue!15}{68.15} & 60.49 & 69.38 & 57.31 \\
        + RULE & \cellcolor{blue!15}{87.84} & \cellcolor{blue!15}{78.00} & \cellcolor{blue!15}{85.78} & \cellcolor{red!25}{83.92} & \cellcolor{blue!15}{87.49} & \cellcolor{blue!15}{83.44} & \cellcolor{blue!15}{87.12} & \cellcolor{red!25}{92.89}& \cellcolor{blue!15}{77.08} & 68.97 & \cellcolor{blue!15}{73.80}& 68.13 & \cellcolor{blue!15}{61.41} & \cellcolor{blue!15}{70.36} & \cellcolor{blue!15}{58.91} \\ 
        \midrule
        \ours\ & 
        \cellcolor{red!25}{89.54}& \cellcolor{red!25}{80.72}& \cellcolor{red!25}{87.13} & \cellcolor{blue!15}{83.57}& \cellcolor{red!25}{88.49}&
        \cellcolor{red!25}{85.08} &
        \cellcolor{red!25}{87.94} &
        \cellcolor{blue!15}{92.78} &
        \cellcolor{red!25}{80.81} & \cellcolor{red!25}{72.95} & \cellcolor{red!25}{76.35} & \cellcolor{red!25}{72.25} & \cellcolor{red!25}{64.54} & \cellcolor{red!25}{73.09} & \cellcolor{red!25}{61.42} \\
    \bottomrule
    \end{tabular}
    }
    \vspace{-1em}
    \label{tab:vqa}
\end{table}
\begin{table}[t]
    \centering
    \footnotesize
    \caption{Model performance (\%) of different methods on report generation task. Notably, we report the average BLEU score, ROUGE-L, METEOR. For detailed BLEU score, see Appendix~\ref{sec:bleu}.}
    \vspace{-1em}
    \resizebox{\linewidth}{!}{
    \begin{tabular}{l|cccccc|ccc}
    \toprule
        \multirow{2}{*}{Models} & \multicolumn{6}{c|}{\textbf{Radiology}} & \multicolumn{3}{c}{\textbf{Ophthalmology}} \\ \cmidrule(r){2-7} \cmidrule(r){8-10}
        & \multicolumn{3}{c}{IU-Xray} & \multicolumn{3}{c}{MIMIC-CXR} & \multicolumn{3}{c}{Harvard-FairVLMed}  \\ \cmidrule(r){2-4} \cmidrule(r){5-7} \cmidrule(r){8-10}
        & BLEU & ROUGE-L & METEOR & BLEU & ROUGE-L & METEOR & BLEU & ROUGE-L & METEOR \\
        \midrule
        LLaVA-Med-1.5 & 9.64&12.26&8.21 &  12.11&13.05&11.16 & 18.11&11.36&10.75  \\
        \midrule
        + Greedy & 11.47&15.38&12.69 & 16.63&14.26&14.19 & 17.98&11.49&13.77 \\
        + Beam Search & 12.10&16.21&13.17 & 16.97&14.74&14.43 & 18.37&12.62&14.50 \\
        + DoLa & 11.79&15.82&12.72 & 17.11&14.89&14.81 & 18.26&12.51&14.51 \\
        + OPERA & 10.66&14.70&12.01 & 15.40&12.52&13.72 & 16.59&11.47&13.63  \\ 
        + VCD & 10.42&14.14&11.59 & 15.18&12.30&13.38 & 16.73&11.38&13.89  \\ \midrule
        + MedDr & 12.37&16.45&13.50 & 18.59&15.72&16.77 & 19.82&13.72&15.40 \\
        + FactMM-RAG & 14.70 & 18.05 & 15.92 & \cellcolor{blue!15}{18.71} & \cellcolor{blue!15}{15.84} & 16.82 & 20.82 & 14.17 & 15.31 \\
        + RULE & \cellcolor{blue!15}{27.53} & \cellcolor{blue!15}{23.16} & \cellcolor{blue!15}{27.99} & 18.61& \cellcolor{red!25}{15.96}& \cellcolor{blue!15}{17.42} & \cellcolor{blue!15}{22.35} & \cellcolor{blue!15}{14.93} & \cellcolor{blue!15}{17.74} \\ 
        \midrule
        \ours\  & \cellcolor{red!25}{31.38} & \cellcolor{red!25}{25.59}&\cellcolor{red!25}{32.43} & \cellcolor{red!25}{23.25} & 12.34 & \cellcolor{red!25}{20.47} & \cellcolor{red!25}{24.82} & \cellcolor{red!25}{16.59}&\cellcolor{red!25}{19.85} \\
    \bottomrule
    \end{tabular}
    }
    \vspace{-2em}
    \label{tab:report}
\end{table}

\noindent \textbf{Comparison with Baselines.} 
We compare \ours\ with baseline methods on medical VQA and report generation tasks, with the results presented in Table~\ref{tab:vqa} and Table~\ref{tab:report}, respectively. Overall, \ours\ outperforms all baselines across nearly all metrics and datasets. Specifically, \ours\ demonstrates a significant performance boost, improving by 18.5\% and 69.1\% over the original Med-LVLM in medical VQA and report generation tasks, respectively. When compared to baseline methods, \ours\ surpasses decoding-based approaches, achieving improvements of 11.5\% and 44.2\% in the two tasks. Furthermore, recent RAG-based methods show substantial improvements over earlier techniques, yet our approach still outperforms RAG-based baselines by 2.8\% and 16.1\% in the medical VQA and report generation tasks, respectively. This indicates that \ours\ effectively mitigates misalignment issues introduced by RAG. Notably, \ours\ achieves more pronounced gains in report generation, likely due to the higher complexity of the task and the greater influence of retrieved contexts in guiding open-ended generation.

\begin{wraptable}{r}{0.33\textwidth}
\begin{center}
\scriptsize
\vspace{-1em}
\caption{Performance comparison with several Med-LVLMs. Rad: Radiology, Opt: Ophthalomology, Pat: Pathology.}
\vspace{-1em}
\label{tab:other}
\begin{tabular}{l|ccc}
\toprule
Model & Rad & Opt & Pat \\\midrule
Med-Flamingo & 27.42 & 22.50 & \underline{29.11}  \\
MedVInT  & 33.17 & \underline{29.40} & 25.33 \\
RadFM & 35.82 & 27.07 & 24.82 \\
miniGPT-Med & \underline{36.66} & 25.28 & 23.16 \\
\ours\ & \textbf{56.94} & \textbf{56.38} & \textbf{54.10}  \\
\bottomrule
\end{tabular}
\end{center}
\vspace{-1.5em}
\end{wraptable}
\noindent
\textbf{Comparison with Other Med-LVLMs.} To provide a comprehensive comparison, we evaluate \ours\ against other open-source Med-LVLMs to demonstrate the superiority of our approach. We assess the performance of these models across different medical image modalities, reporting the average results for medical VQA and report generation tasks in Table~\ref{tab:other} (see Appendix~\ref{sec:result} for detailed results). Our findings show that \ours\ significantly outperforms Med-LVLMs pre-trained on large-scale datasets across various domains. This reinforces the generalizability and effectiveness of our approach across diverse image domains and medical multimodal tasks.

\vspace{-0.5em}
\subsection{Analysis}
\vspace{-0.5em}
In this section, we provide a detailed analysis of each module's performance, along with a series of analytical experiments, to better understand the performance gains of \ours. 
Additionally, we demonstrate the compatibility of our method in Appendix~\ref{sec:result}, including its application to generalist and domain-specific Med-LVLMs. 

\begin{wraptable}{r}{0.45\textwidth}
\begin{center}
\scriptsize
\vspace{-1em}
\caption{Ablation results on two datasets covering different domains. RG: report generation, FairVLMed: Harvard-FairVLMed.}
\vspace{-0.5em}
\label{tab:aba}
\begin{tabular}{l|ccccc}
\toprule
Model & \multicolumn{2}{c}{IU-Xray} & \multicolumn{2}{c}{FairVLMed} \\
& VQA & RG & VQA & RG \\
\midrule
LLaVA-Med-1.5 & 68.99 & 10.04 & 66.63 & 13.41 \\
+DR & 77.12 & 13.23 & 72.69 & 15.89 \\
+RCS & 79.56 & 17.92 & 75.74 & 17.22 \\
+RAG-PT (Ours) & \textbf{85.80} & \textbf{29.80} & \textbf{87.18} & \textbf{20.42} \\
\bottomrule
\end{tabular}
\end{center}
\vspace{-1em}
\end{wraptable}
\textbf{Ablation Studies.} We conduct a series of ablation experiments to evaluate the impact of each component in \ours. The results for both medical VQA and report generation tasks on the IU-Xray and Harvard-FairVLMed datasets are summarized in Table~\ref{tab:aba}. According to the results, we can see that: (1) The domain-aware retrieval mechanism (DR) significantly improves the factuality of Med-LVLM, with an average performance increase of 17.9\% and 16.1\% on the IU-Xray and FairVLMed datasets, respectively. Here, the retrieved knowledge aids the model in generating more factual responses. (2) Building on this, the introduction of adaptive retrieval context selection (RCS) further filters out unreliable retrieved contexts, yielding an additional performance boost of 19.3\% and 6.3\% on the IU-Xray and FairVLMed datasets. (3) The inclusion of RAG-based preference fine-tuning (RAG-PT) enhances the model's understanding of the retrieved knowledge, leading to substantial performance gains of 37.1\% and 16.9\% on the respective datasets. This demonstrates that RAG-PT effectively addresses misalignment issues. 

\begin{wraptable}{r}{0.45\textwidth}
\begin{center}
\scriptsize
\vspace{-2.5em}
\caption{Performance using RAG-PT based on subsets of preference data.}
\vspace{-0.5em}
\label{tab:pre}
\begin{tabular}{l|ccccc}
\toprule
Model & \multicolumn{2}{c}{IU-Xray} & \multicolumn{2}{c}{FairVLMed} \\
& VQA & RG & VQA & RG \\
\midrule
LLaVA-Med-1.5 & 68.99 & 10.04 & 66.63 & 13.41 \\
+RAG-PT 1 & 80.19 & 19.38 & 79.42 & 18.37 \\
+RAG-PT 2 & 80.27 & \textbf{20.16} & 79.35 & 18.66 \\
+RAG-PT 3 & \textbf{81.30} & 19.43 & \textbf{80.07} & \textbf{18.92} \\
\bottomrule
\end{tabular}
\end{center}
\vspace{-1em}
\end{wraptable}
\noindent
\textbf{Impact of the Preference Data in RAG-PT.} To better understand how RAG-PT mitigates the misalignment issue and improves performance, we conducted a detailed study on the training preference data composition of RAG-PT. As described in Section~\ref{sec:rag-pt}, the RAG-PT data is designed to address both cross-modality alignment and overall alignment objectives, with the latter focusing on enhanced understanding of retrieved knowledge and minimizing retrieval interference. The detailed experimental results in Table~\ref{tab:pre} demonstrate that the preference data tailored for different alignment objectives positively impacts the model's performance, showing the effectiveness of RAG-PT. Additional ablation analysis on preference data can be seen in Appendix~\ref{sec:abl_more}.

\noindent
\textbf{How Effective is \ours\ in Mitigating Misalignment Issues?} 
To gain a more intuitive understanding of the effectiveness of \ours\ in addressing misalignment issues: 1) we calculate the proportion of errors caused by RAG and compare it to the proportion after incorporating \ours. 2) We visualize the attention maps of image and text tokens with and without RAG-PT. 
First, as mentioned in Section~\ref{sec:rag-pt}, the model may directly copy reference information, referred to as Copy-Reference (CR) rate. After applying \ours, as shown in Figure~\ref{fig:alignment}, the CR rate drops to 28.19\%. Additionally, the proportion of errors affected by RAG interference, referred to as Over-Reliance (OR) rate, which is initially 43.31\%, decreased to 8.38\% after incorporating \ours.
Furthermore, as shown in Figure~\ref{fig:jm_acc}, the original Med-LVLM tends to rely more heavily on text while ignoring visual information. When retrieval information is introduced, the original Med-LVLM focused more on the retrieved answers, even if the content is incorrect. After RAG-PT, the model significantly increases its attention to visual information and reduces the interference of RAG, thus better aligning the model's knowledge with the fundamental facts.

\hfill
\begin{minipage}[t]{0.3\textwidth}
\vspace{-8em}
    \centering
        \includegraphics[width=\textwidth{}]{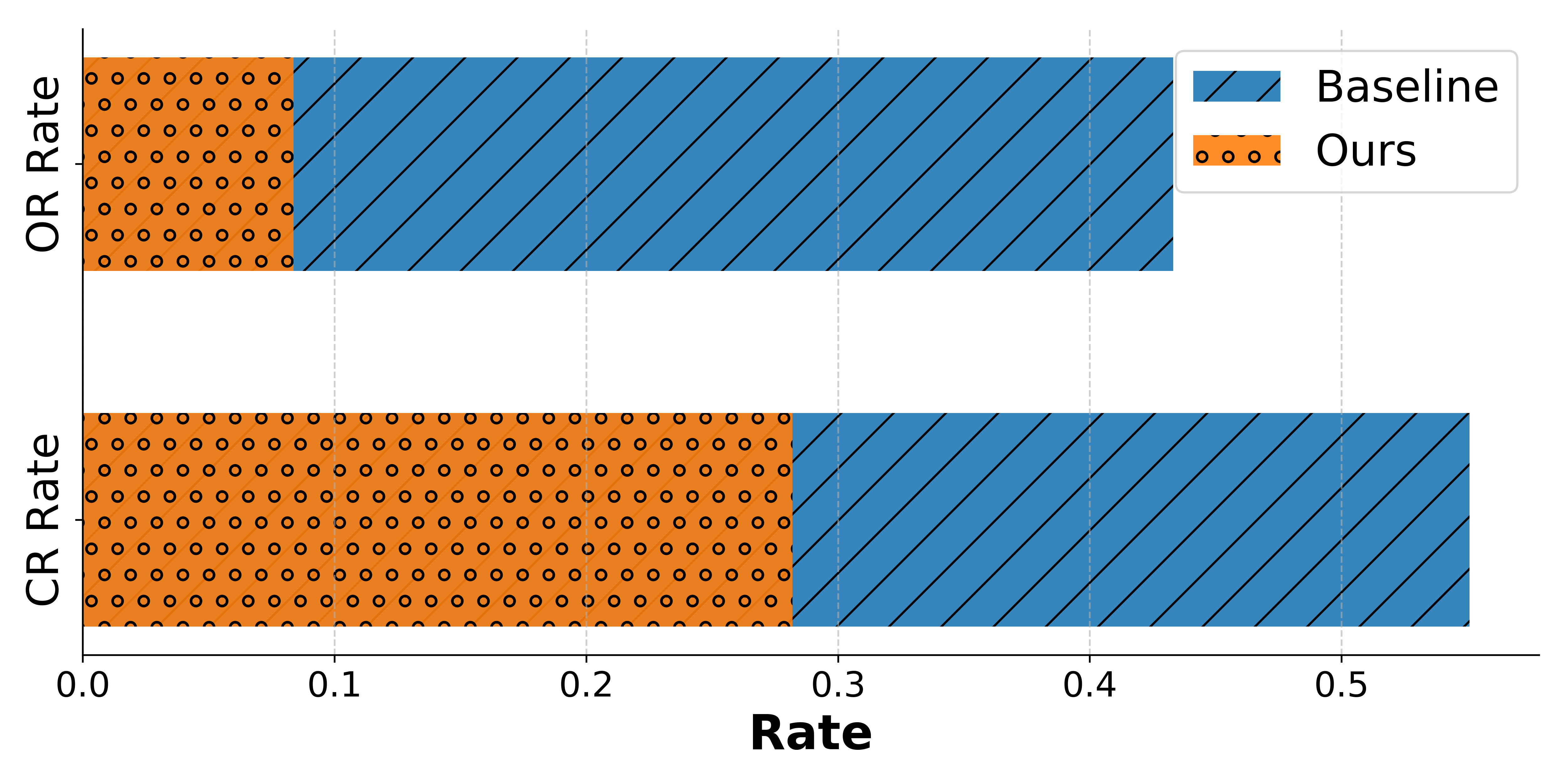}
        \captionof{figure}{Alignment analysis with and without RAG. OR: Over-Reliance; CR: Copy-Reference.}
    \label{fig:alignment}
\end{minipage}
\hfill
\begin{minipage}[t]{0.66\textwidth}
    \centering
       \includegraphics[width=\textwidth{}]{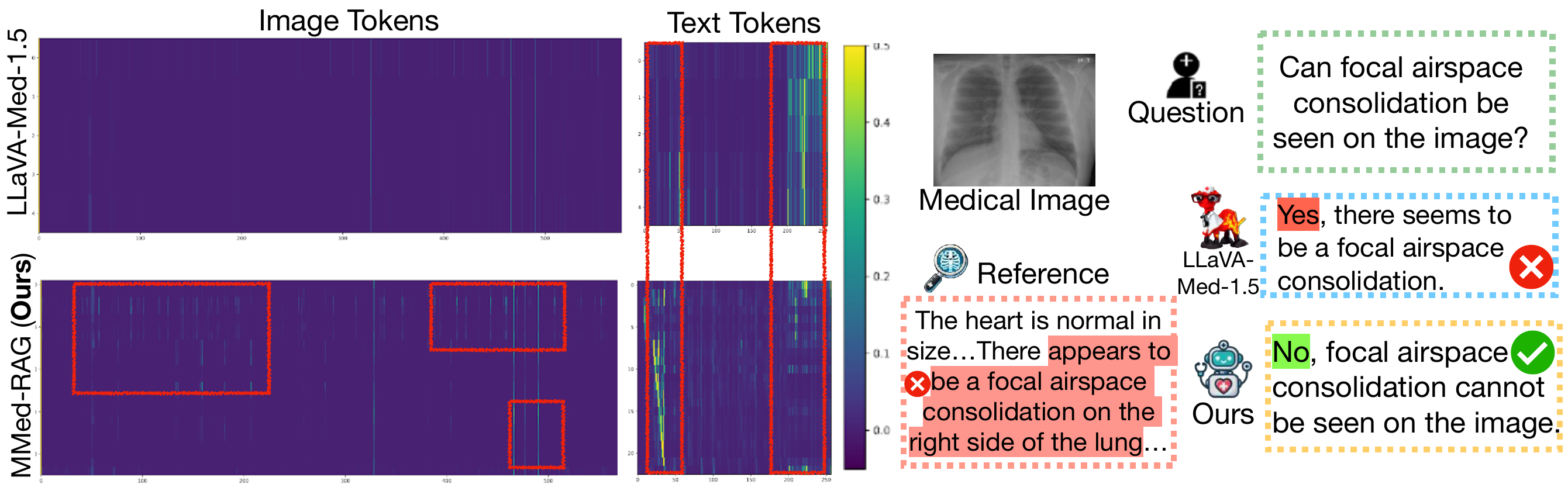}
        \captionof{figure}{Visualization of attention map. The red box region is labeled with the attentions that can be enhanced by \ours.}
    \label{fig:jm_acc}
\end{minipage}




\vspace{-0.5em}
\section{Related Work}
\textbf{Factuality in Med-LVLMs.} The rapid advancements in Large Vision-Language Models (LVLMs)~\citep{liu2023improved,liu2023visual} are beginning to influence the field of medical image analysis. Several Med-LVLMs~\citep{li2023llava,moor2023med,zhang2023pmc,wu2023towards,zhu2024mmedpo}, have emerged, showing remarkable performance across different medical imaging modalities. Despite these advances, Med-LVLMs continue to present notable factual hallucination~\citep{xia2024cares,royer2024multimedeval}, generating textual outputs that contradict medical visual information. This raises concerns about potential misdiagnoses or overlooked conditions. Recently, benchmarks have been developed to assess the accuracy of Med-LVLMs in tasks such as visual question answering (VQA) and report generation~\citep{xia2024cares,royer2024multimedeval}. However, research aimed at enhancing the factual accuracy of Med-LVLMs remains relatively unexplored.

\noindent
\textbf{Retrieval Augmented Generation in Med-LVLMs.}  Retrieval-Augmented Generation (RAG) has proven to be a powerful technique for enhancing factual accuracy in language modeling~\citep{gao2023retrieval,wu2023ragtruth,chen2024mllm,qu2024alleviating,sun2024surf}. In the biomedical domain, RAG leverages external knowledge to guide the generation of Med-LVLMs, offering clear advantages in tasks such as medical VQA and report generation~\citep{yuan2023ramm,kumar2024improving,tao2024memory,he2024meddr,sun2024fact}. However, these works mainly focus on enhancing the relevance of the retrieved contexts without considering the model's understanding of retrieved knowledge. There are several recent work on RAG fine-tuning in LLMs. DPA-RAG~\citep{dong2024understand} addresses the alignment issues between the external reranker and the internal LLM through supervised fine-tuning. Then RAG-DDR~\citep{li2024rag} leverages a rolling method to generate perturbed responses, further mitigating conflicts between parameter memory and external knowledge. In the biomedical domain, RULE~\citep{xia2024rule} is proposed to use preference fine-tuning to reduce the model's over-reliance on retrieved contexts. However, it still overlooks misalignment issues caused by RAG, as well as the generalizability of the retriever given the diverse domains of input images. In response, we propose \ours\ to mitigate these risks, enhancing the factuality of Med-LVLMs by addressing these overlooked factors. This can lead to a better cross-modality and overall alignment to enhance the understanding of retrieved knowledge and visual information, ensuring more consistent and reliable performance across tasks.

\vspace{-0.5em}
\section{Conclusion}
This paper introduces \ours, a versatile multimodal RAG system designed to address the critical issue of factual hallucination in Med-LVLMs. \ours\ employs a domain-aware retrieval mechanism, adaptive calibration for selecting the optimal number of retrieved contexts, and RAG-based preference fine-tuning to improve both cross-modality alignment and overall alignment with the ground truth. These enhancements significantly boost the factual accuracy of Med-LVLMs. Experimental results demonstrate the effectiveness of \ours\ in enhancing factual accuracy across various imaging domains, underscoring its potential for reliable use in healthcare. Our findings underscore the importance of incorporating robust multimodal RAG mechanism to ensure that Med-LVLMs can serve as dependable tools in clinical settings.
\section*{Ethics Statement}
This paper presents a novel RAG-based approach to enhancing the factuality of Med-LVLMs. 
We have followed best practices in data collection, model design, and evaluation, ensuring adherence to privacy and ethical standards. All datasets used are sourced from publicly available medical datasets or collected with appropriate ethical considerations, including patient data anonymization. 
We adhere to principles of research integrity and transparency, and comply with all relevant regulations. We hope that our research will contribute to safer, more reliable AI-assisted medical tools and advance healthcare technology responsibly.

\section*{Reproducibility Statement}
We have taken significant steps to ensure that our work is reproducible. All details regarding our proposed multimodal RAG system, including the domain-aware retrieval mechanism, adaptive retrieved context selection, and RAG-based preference fine-tuning strategy, are described comprehensively in Section~\ref{sec:3}. We provide the hyperparameter settings and experimental configurations used in our evaluations in Section~\ref{sec:exp} and Appendix~\ref{sec:setting}. Additionally, we have included detailed pseudocode for the proposed algorithms in Algorithm~\ref{ag:dpo} and an in-depth explanation of the data processing steps for each medical dataset used in Appendix~\ref{sec:detail} and Appendix~\ref{sec:app_data}. 

\section*{Acknowledgement}
This work is partially supported by Cisco Faculty Research Award.
\bibliography{main}
\bibliographystyle{iclr2025_conference}

\appendix
\section{Experiment}
\subsection{Experimental Setup}
\subsubsection{Data Statistics}
\label{sec:detail}
The data quantities used in this study are presented in Table~\ref{tab:data1}, Table~\ref{tab:data2} and Table~\ref{tab:data3}. We clarify that for training the retriever, the data refers to the number of image-text pairs, while for fine-tuning, it refers to the number of QA items. The “All” category represents the total amount of data used to construct the preference dataset for RAG-PT. The training of RAG-PT includes three types of samples: (a) clean samples with originally correct answers that remain correct even after adding noise to the images, (b) clean image samples with originally incorrect answers that become correct, and (c) clean image samples with originally correct answers that become incorrect.

\begin{table}[htbp]
    \centering
    \footnotesize
    \caption{Data statistics for medical VQA task. "Train (DR)" refers to the number of image-text pairs for retriever training, "All (RAG-PT)" refers to the total data for RAG-PT, and "Train (RAG-PT)-a/b/c" refer to the respective subsets for RAG-PT training.}
    \vspace{-1em}
    \resizebox{\linewidth}{!}{
    \begin{tabular}{l|cccccc}
    \toprule
    Dataset & Train (DR) & All (RAG-PT) & Train (RAG-PT)-a & Train (RAG-PT)-b & Train (RAG-PT)-c \\ \midrule
    Ophthalomology & 7000 & 3247 & 1082 & 1030 & 1135 \\
    Radiology & 4034 & 4836 & 1612 & 1989 & 1235 \\
    Pathology & 5000 & 1990 & 663 & 523 & 804 \\
    \bottomrule
    \end{tabular}
    }
    \label{tab:data1}
\end{table}

\begin{table}[htbp]
    \centering
    \footnotesize
    \caption{Data statistics for report generation. "Train (DR)" refers to the number of image-text pairs for retriever training, "All (RAG-PT)" refers to the total data for RAG-PT, and "Train (RAG-PT)-a/b/c" refer to the respective sample categories for RAG-PT training.}
    \vspace{-1em}
    \begin{tabular}{l|cccccc}
    \toprule
    Dataset & Train (R) & All (RAG-PT) & Train (RAG-PT)-a & Train (RAG-PT)-b & Train (RAG-PT)-c \\ \midrule
    Ophthalmology & 7000 & 3247 & 142 & 78 & 207 \\
    Radiology & 4034 & 4836 & 233 & 126 & 342 \\
    \bottomrule
    \end{tabular}
    \label{tab:data2}
\end{table}

\begin{table}[htbp]
    \centering
    \footnotesize
    \caption{Data statistics for various datasets. The rows represent the number of images and QA pairs for each dataset.}
    \vspace{-1em}
    \begin{tabular}{l|ccccc}
    \toprule
    & Harvard-FairVLMed & IU-Xray & MIMIC-CXR & PMC-OA & Quilt-1M \\ \midrule
    \# Images & 713 & 589 & 700 & 530 & 559 \\
    \# QA Items & 4285 & 2573 & 3470 & 3124 & 1994 \\
    \bottomrule
    \end{tabular}
    \label{tab:data3}
\end{table}

\subsubsection{Hyperparameter Settings}
\label{sec:setting}
Following the settings of CLIP~\citep{radford2021learning}, we adopt the same architecture and hyperparameters for the vision and text encoders. The vision encoder is a ResNet-50~\citep{he2016deep}, and the text encoder is a bio-bert-based model~\citep{alsentzer2019publicly}. We use the AdamW optimizer with a learning rate of $10^{-4}$ and a batch size of 512. The model is trained for 360 epochs. For the first phase, we trained for 3 epochs, and for the second phase, the training was conducted for 1 epoch. Training for 20 hours on one A100 80G GPU. For the RAG-PT phase, we adjust the diffusion noise level, symbolized by $\xi$ through a specific formula: $
\xi = \text{Sigmoid}(l_t) \times (0.5 \times 10^{-2} - 10^{-5}) + 10^{-5},$
where $\epsilon$ is drawn from a normal distribution. The reports available for retrieval are from the training set of the corresponding dataset. In our experiments, we apply cross-validation to tune all hyperparameters with grid search. All the experiments are implemented on PyTorch 2.1.2 using four NVIDIA RTX A6000 GPUs. It takes roughly 3 and 4 hours for fine-tuning CLIP and LLaVA-Med-1.5 7B, respectively.

\subsection{Evaluated Datasets}
\label{sec:app_data}
We utilize five open-source medical vision-language datasets, i.e., MIMIC-CXR~\citep{johnson2019mimic}, IU-Xray~\citep{demner2016preparing}, Harvard-FairVLMed~\citep{luo2024fairclip}, PMC-OA~\citep{lin2023pmc} and Quilt-1M~\citep{ikezogwo2024quilt}.
\begin{itemize}[leftmargin=*]
    \item \textbf{MIMIC-CXR}~\citep{johnson2019mimic} is a large publicly available dataset of chest X-ray images in DICOM format with associated radiology reports. 
    \item \textbf{IU-Xray}~\citep{demner2016preparing} is a dataset that includes chest X-ray images and corresponding diagnostic reports. 
    \item \textbf{Harvard-FairVLMed}~\citep{luo2024fairclip} focuses on fairness in multimodal fundus images, containing image and text data from various sources. It aims to evaluate bias in AI models on this multimodal data comprising different demographics. 
    \item \textbf{PMC-OA}~\citep{lin2023pmc} is a large-scale dataset comprising figure-caption pairs extracted from PubMed Central. It covers 2,478,267 papers and includes a total of 12,211,907 figure-caption pairs. We only use the pathology subset filtered by GPT-4 based on the captions.
    \item \textbf{Quilt-1M}~\citep{ikezogwo2024quilt} is the largest vision-language dataset in histopathology, containing 1 million image-text pairs sourced from platforms such as YouTube, Twitter, research papers, and other parts of the internet.
\end{itemize}

\subsection{Evaluated Models}
We evaluate five open-source Med-LVLMs, \textit{i.e.}, LLaVA-Med~\citep{li2023llava}, Med-Flamingo~\citep{moor2023med}, MedVInT~\citep{zhang2023pmc}, RadFM~\citep{wu2023towards}, miniGPT-Med~\citep{alkhaldi2024minigpt}. The selected models are all at the 7B level.
\begin{itemize}[leftmargin=*]
    \item \textbf{LLaVA-Med}~\citep{li2023llava} is a vision-language conversational assistant, adapting the general-domain LLaVA~\citep{liu2023visual} model for the biomedical field. The model is fine-tuned using a novel curriculum learning method, which includes two stages: aligning biomedical vocabulary with figure-caption pairs and mastering open-ended conversational semantics. It demonstrates excellent multimodal conversational capabilities.
    \item \textbf{Med-Flamingo}~\citep{moor2023med} is a multimodal few-shot learner designed for the medical domain. It builds upon the OpenFlamingo, continuing pre-training with medical image-text data from publications and textbooks. This model aims to facilitate few-shot generative medical visual question answering, enhancing clinical applications by generating relevant responses and rationales from minimal data inputs.
    \item \textbf{RadFM}~\citep{wu2023towards} serve as a versatile generalist model in radiology, distinguished by its capability to adeptly process both 2D and 3D medical scans for a wide array of clinical tasks. It integrates ViT as visual encoder and a perceiver module, alongside the MedLLaMA language model, to generate sophisticated medical insights for a variety of tasks. This design allows RadFM to not just recognize images but also to understand and generate human-like explanations.
    \item \textbf{MedVInT}~\citep{zhang2023pmc}, which stands for Medical Visual Instruction Tuning, is designed to interpret medical images by answering clinically relevant questions. This model features two variants to align visual and language understanding: MedVInT-TE and MedVInT-TD. Both MedVInT variants connect a pre-trained vision encoder ResNet-50 adopted from PMC-CLIP~\citep{lin2023pmc}, which processes visual information from images. It is an advanced model that leverages a novel approach to align visual and language understanding.  
    \item \textbf{miniGPT-Med}~\citep{alkhaldi2024minigpt} is a vision-language model derived from large-scale language models and tailored for radiology diagnosis applications. It handles various medical vision-language task using distinct task identifiers, demonstrating advanced performance in disease grounding, medical report generation, and medical VQA.
\end{itemize}

\subsection{Overview of the Baselines}
We compare \ours\ with two types of LVLM hallucination mitigation methods that show promising results in natural image understanding. 1) Decoding-based methods, including Greedy Decoding, Beam Search~\citep{sutskever2014sequence},  DoLa~\citep{chuang2023dola}, OPERA~\citep{huang2023opera}, VCD~\citep{leng2023mitigating}. These methods manipulate the logits of the model's output tokens to enhance factual accuracy. 2) Multimodal RAG-based methods, including MedDr~\citep{he2024meddr}, FactMM-RAG~\citep{sun2024fact}, RULE~\citep{xia2024rule}.
\begin{itemize}[leftmargin=*]
    \item \textbf{Greedy decoding} involves selecting the most probable next token at each step of generation. While it is efficient and straightforward, it can lead to suboptimal outcomes by getting stuck in repetitive or less creative patterns.
    \item \textbf{Beam search}~\citep{sutskever2014sequence} expands on greedy decoding by maintaining multiple candidate sequences (or "beams") at each step, allowing for a broader exploration of possible outputs. This approach balances quality and diversity by selecting the top-k sequences based on their probabilities, resulting in more coherent and creative text generation compared to greedy decoding.
    \item \textbf{DoLa}~\citep{chuang2023dola} derives the next-token distribution by contrasting the logits projected from later layers against those from earlier layers, leveraging the fact that factual knowledge in LLMs is typically localized within specific transformer layers.
    \item \textbf{OPERA}~\citep{huang2023opera} is a LVLMs decoding method based on an Over-trust Penalty and a Retrospection-Allocation strategy The key insight is that hallucinations are closely tied to knowledge aggregation patterns in the self-attention matrix, where MLLMs tend to focus on summary tokens, neglecting image tokens and resulting in content hallucination. 
    \item \textbf{VCD}~\citep{leng2023mitigating} is a decoding method that tackles the object hallucination issue in LVLMs. It contrasts output distributions derived from original and distorted visual inputs to calibrate the model’s output without the usage of external tools, reducing the the over-reliance on statistical bias and unimodal priors.
    \item \textbf{MedDr}~\citep{he2024meddr} is a healthcare foundation model built upon generated diagnosis-based datasets, demonstrating advanced capabilities in various data modalities. Meddr also integrates a retrieval-augmented medical diagnosis strategy during inferencing to enhance factual accuracy.
    \item \textbf{FactMM-RAG}~\citep{sun2024fact} is a fact-aware multimodal retrieval-augmented pipeline for radiology report generation. It utilize RadGraph to annotate chest radiograph reports and mine clinically relevant pairs to train a universal multimodal retriever.\
    \item \textbf{RULE}~\citep{xia2024rule} is an advanced medical retrieval-augmented generation strategy designed to enhance the factuality of Med-LVLMs. First, it introduces a robust strategy for controlling factuality risk through the calibrated selection of retrieved contexts. Second, RULE develops a preference optimization strategy to balance Med-LVLMs’ intrinsic knowledge and the retrieved information.
\end{itemize}

\subsection{Prompts}
We convert the medical reports into a series of closed-ended questions with \textit{yes} or \textit{no} answers. To ensure the quality of the VQA data, we perform a round of self-checks using GPT-4~\citep{openai2023gpt4}. Finally, we conduct an round of manual filtering to remove questions with obvious issues or those related to multiple images or patient histories. The prompt templates used are shown in Table~\ref{tab:prompt}.
\begin{table}[t]
    \centering
    \footnotesize
    \setlength{\arrayrulewidth}{0.5mm}
    \definecolor{mygray}{gray}{0.93}
    \begin{tabular}{|>{\columncolor{mygray}}p{12.5cm}|}
    \hline
    \textbf{Instruction [Round1]} \\
    You are a professional medical expert. I will provide you with some medical reports. Please generate some questions with answers (the answer should be yes or no) based on the provided report. The subject of the questions should be the medical image or patient, not the report. \\ Below are the given report: \\ {[REPORT]} \\
    \textbf{Instruction [Round2]} \\
    Please double-check the questions and answers, including how the questions are asked and whether the answers are correct. You should only generate the questions with answers and no other unnecessary information. \\
    Below are the given report and QA pairs in round1: \\
    {[REPORT]} \\  {[QA PAIRS R1]}\\ \hline
    \end{tabular}
    \caption{The instruction to GPT-4 for generating QA pairs.}
    \label{tab:prompt}
\end{table}

\subsection{Additional Results}
\label{sec:result}

\subsubsection{Compatibility Analysis} 
\label{sec:llava-1}
To demonstrate the compatibility of our approach across different backbone models, we apply it to LLaVA-Med-1.0. As shown in Table~\ref{tab:backbone}, our method delivers an average improvement of 40.3\% over the original LLaVA-Med-1.0, further highlighting its effectiveness in enhancing RAG performance and its adaptability to various backbones. \ours\ can be transferred to different Med-LVLMs, yielding consistent improvements across various domains, demonstrating the compatibility of our method.
\begin{table}[htbp]
\begin{center}
\footnotesize
\vspace{-1em}
\caption{Performance on different backbones. }
\vspace{-0.5em}
\label{tab:backbone}
\begin{tabular}{l|ccccc}
\toprule
Model & \multicolumn{2}{c}{IU-Xray} & \multicolumn{2}{c}{FairVLMed} \\
& VQA & RG & VQA & RG \\
\midrule
LLaVA-Med-1.0 & 61.73 & 8.74 & 59.54 & 10.59 \\
+\ours\ & \textbf{80.32} & \textbf{22.63} & \textbf{78.49} & \textbf{15.88}  \\
\bottomrule
\end{tabular}
\end{center}
\vspace{-1em}
\end{table}

\subsubsection{Detailed Results of Other LVLMs} As shown in Table~\ref{tab:app_vqa}, we conduct a comparison of several general LVLMs and other Med-LVLMs, including GPT-4o~\citep{openai2023gpt4}, Gemini-1.5~\citep{reid2024gemini}, QwenVL~\citep{bai2023qwenvl}, LLaVA-1.6~\citep{liu2023visual}, and InternVL-2~\citep{chen2024far}. Our findings show that \ours\ consistently outperforms these models, further demonstrating its effectiveness. 

\begin{table}[htbp]
    \centering
    \footnotesize
    \caption{Accuracy (\%) of different Med-LVLMs based on LLaVA-Med-1.5 on medical VQA task. }
    \resizebox{\linewidth}{!}{
    \begin{tabular}{l|cc|c|cc}
    \toprule
        \multirow{2}{*}{Models} & \multicolumn{2}{c|}{\textbf{Radiology}} &\textbf{Ophthalmology} & \multicolumn{2}{c}{\textbf{Pathology}} \\  
        & IU-Xray & MIMIC-CXR & Harvard-FairVLMed & Quilt-1M & PMC-OA (Pathology) \\ \midrule
        LLaVA-Med-1.5 & 75.47 & 75.79 & 63.03& 62.80& 59.28  \\ 
        \ours\ & 
        \textbf{89.54}& \textbf{83.57} & 
        \textbf{87.94} &
        \textbf{72.95} & \textbf{64.54} \\
        \midrule
        Med-Flamingo & 26.74 & 61.27 & 42.06& 27.11& 32.62 \\ 
        MedVInT & 73.34 & 66.06& 35.92 & 26.81 & 27.77 \\
        RadFM & 26.67 &  69.30 &  52.47 & 27.02 & 25.12 \\
        miniGPT-Med & 54.87 & 53.92 & 66.73 & 26.82 & 27.03 \\
        \midrule
        GPT-4o & 63.25 & 60.61 & 61.50 & 53.56 & 49.70 \\
        Gemini-1.5 & 59.73 & 61.02 & 58.53 & 56.88 & 52.17 \\
        LLaVA-v1.6	& 58.05 & 63.70 & 48.52 & 35.73 & 38.54 \\
        Qwen-VL-Chat & 59.43 & 60.43 & 38.06 & 28.74 & 29.53 \\
        InternVL-2 & 54.06 & 59.47 & 44.38 & 37.82 & 34.40 \\
    \bottomrule
    \end{tabular}
    }
    \label{tab:app_vqa}
\end{table}

\subsubsection{Comparison with Domain-Specific Med-LVLMs and them with RAG-PT}
We conduct experiments to compare our method with domain-specific Med-LVLMs as follows: Radiology: RadFM~\citep{wu2023towards}, Pathology: Quilt-LLaVA~\citep{seyfioglu2024quilt}, Ophthalmology: RetinaVLM~\citep{holland2024specialist}. For radiology, we use the IU-Xray dataset to evaluate VQA. For pathology, we use the PMC-OA pathology subset to evaluate VQA. For ophthalmology, since the domain-specific Med-LVLM, i.e., RetinaVLM, is only trained on report-generation tasks, we use the Harvard-FairVLMed dataset to evaluate report generation. As shown in Table~\ref{tab:domain_comparison}, our method significantly outperforms each domain-specific Med-LVLM. Additionally, we apply RAG-PT to each domain-specific Med-LVLM. As shown in Table~\ref{tab:domain_comparison}, after incorporating RAG-PT, the performance of these models improve significantly, demonstrating the compatibility of our method. Furthermore, domain-specific Med-LVLMs could outperform generalist Med-LVLMs in their specialized domains, as they are fine-tuned using specialized medical domain data. While this significantly enhances their medical understanding in specific domains, it may reduce their generalization ability, such as their capacity to comprehend retrieved information. Consequently, even after incorporating RAG-PT, the performance of several domain-specific Med-LVLMs (e.g., RetinaVLM and RadFM) is inferior to MMed-RAG.

\begin{table}[htbp]
\centering
\footnotesize
\caption{Model performance comparison with domain-specific Med-LVLMs.}
\label{tab:domain_comparison}
\resizebox{0.9\linewidth}{!}{
\begin{tabular}{@{}lccccccccc@{}}
\toprule
\multirow{2}{*}{Model} &
  \multicolumn{3}{c}{Radiology} &
  \multicolumn{3}{c}{Pathology} &
  \multicolumn{3}{c}{Ophthalmology} \\ \cmidrule(lr){2-4} \cmidrule(lr){5-7} \cmidrule(lr){8-10}
 &
  Acc &
  F1 &
  AUC &
  Acc &
  F1 &
  AUC &
  BLEU &
  ROUGE-L &
  METEOR \\ \midrule
RadFM        & 26.67 & 30.36 & 55.31 & -     & -     & -     & -     & -     & -     \\
+ RAG-PT      & 48.39 & 39.40 & 59.70 & -     & -     & -     & -     & -     & -     \\ \midrule
Quilt-LLaVA   & -     & -     & -     & 62.59 & 72.30 & 56.96 & -     & -     & -     \\
+ RAG-PT      & -     & -     & -     & \textbf{64.72} & \textbf{73.36} & \underline{61.39} & -     & -     & -     \\ \midrule
RetinaVLM     & -     & -     & -     & -     & -     & -     & 19.96 & 12.73 & 13.52 \\
+ RAG-PT      & -     & -     & -     & -     & -     & -     & \underline{22.26} & \underline{14.64} & \underline{16.87} \\ \midrule
LLaVA-Med-1.5 & \underline{75.47} & 64.04 & 67.46 & 59.28 & 71.98 & 54.19 & 18.11 & 11.36 & 10.75 \\
MMed-RAG      & \textbf{84.10} & \textbf{71.92} & 86.40 & \underline{64.54} & \underline{73.09} & \textbf{61.42} & \textbf{24.82} & \textbf{16.59} & \textbf{19.85} \\ \bottomrule
\end{tabular}
}
\end{table}

\subsubsection{Results on Other Domain}
\label{sec:ood}
We apply RAG-PT to one additional domain (i.e., environmental ecosystems modeling) to further validate the effectiveness of RAG-PT. We conduct experiments on two environmental system modeling datasets~\citep{li2024lite}. The CRW-Temp dataset is a river water temperature prediction dataset aimed at forecasting the daily average water temperature of a specific day based on observed physical variables. The CRW-Flow dataset focuses on predicting river segment flow based on observed physical variables. The model used is LITE~\citep{li2024lite}, an environmental system large model based on LLaMA2~\citep{touvron2023llama}. We train a semantic time-series encoder using time-series information-text pairs, which works in conjunction with a text encoder as the retriever. Then we retrieve the most similar environmental descriptions based on the current environmental descriptions. As shown in Table~\ref{tab:ood}, our approach demonstrates significant performance improvements on tasks in this domain.

\begin{table}[htbp]
\centering
\caption{Performance comparison of different models on CRW-Temp and CRW-Flow datasets.}
\footnotesize
\label{tab:ood}
\begin{tabular}{@{}lcccc@{}}
\toprule
\multirow{2}{*}{Model} & \multicolumn{2}{c}{CRW-Temp} & \multicolumn{2}{c}{CRW-Flow} \\ \cmidrule(lr){2-3} \cmidrule(lr){4-5}
                       & RMSE        & MAE        & RMSE        & MAE        \\ \midrule
LITE~\citep{li2024lite}                   & 2.02        & 1.70       & 2.39        & 1.02       \\
+RAG                   & 1.93        & 1.62       & 2.27        & 0.96       \\
+RAG-PT                  & \textbf{1.74}        & \textbf{1.46}       & \textbf{2.11}        & \textbf{0.90}       \\ \bottomrule
\end{tabular}\end{table}

\subsubsection{Statistics of Copy-Reference Rate and Over-Reliance Rate for more LVLMs.}
Following the alignment analysis method we apply to LLaVA-Med-1.5 in Section~\ref{sec:rag-pt}, we conduct two alignment analysis tests on multiple open-source Med-LVLMs and commercial LVLMs using the Harvard-FairVLMed dataset with the incorporation of retrieved information. These tests respectively evaluate (1) cross-modality alignment and (2) overall alignment with the ground truth. As shown in Table~\ref{tab:copy_over_reliance}, the results indicate that both existing open-source Med-LVLMs and commercial LVLMs exhibit misalignment issues with retrieved information. In addition, it is worthwhile to mention that GPT-4o demonstrates the best alignment performance compared with other models when incorporating RAG, especially in cross-modal alignment. This is likely because GPT-4o has been well-trained in visual perception and may also have utilized some post-training methods (like preference optimization) to optimize modal alignment. 
\begin{table}[htbp]
\centering
\caption{Comparison of Copy-Reference Rate and Over-Reliance Rate across different models.}
\label{tab:copy_over_reliance}
\footnotesize
\begin{tabular}{@{}lcc@{}}
\toprule
Model            & Copy-Reference Rate & Over-Reliance Rate \\ \midrule
LLaVA-Med-1.5    & 55.08               & 43.31              \\
Med-Flamingo     & 60.17               & 33.74              \\
miniGPT-Med      & 56.75               & 46.06              \\
GPT-4o           & 12.54               & 24.80              \\ \bottomrule
\end{tabular}
\end{table}

\subsubsection{Detailed Ablation Analysis}
\label{sec:abl_more}
Preference data designed for different alignment objectives can indeed produce varying effects. Therefore, conducting ablation experiments on combinations of different types of preference data is necessary. We perform comprehensive ablation experiments on RAG-PT 1/2/3 as well as their combinations (RAG-PT 1+2, 2+3, 1+3) to analyze the effectiveness of each type of data and their combinations. We find that the combination of 1+3 produced the most significant results, indicating that the two misalignment issues (i.e., cross-modality and over-reliance issues) are the most prominent. Targeted mitigation of these two issues yielded the greatest improvement. However, incorporating data for all three alignment objectives yields the best performance, demonstrating the importance of each alignment component.

\begin{table}[htbp]
\begin{center}
\footnotesize
\vspace{-1em}
\caption{Ablation results using RAG-PT based on subsets of preference data.}
\label{tab:pre-abl}
\begin{tabular}{l|ccccc}
\toprule
Model & \multicolumn{2}{c}{IU-Xray} & \multicolumn{2}{c}{Harvard-FairVLMed} \\
& VQA & RG & VQA & RG \\
\midrule
LLaVA-Med-1.5 & 68.99 & 10.04 & 66.63 & 13.41 \\
+RAG-PT 1 & 80.19 & 19.38 & 79.42 & 18.37 \\
+RAG-PT 2 & 80.27 & 20.16 & 79.35 & 18.66 \\
+RAG-PT 3 & 81.30 & 19.43 & 80.07 & 18.92 \\ \midrule
+RAG-PT 1+2 & 82.58 & 22.74 & 82.08 & 18.97 \\
+RAG-PT 1+3 & 82.96 & 24.50 & 82.87 & 19.22 \\
+RAG-PT 2+3 & \underline{83.61} & \underline{25.77} & \underline{83.89} & \underline{19.30} \\ \midrule
+RAG-PT 1+2+3 & \textbf{85.58} & \textbf{29.69} & \textbf{87.02} & \textbf{20.31} \\
\bottomrule
\end{tabular}
\end{center}
\vspace{-1em}
\end{table}

\subsubsection{External Validation}
\label{sec:ext}
Considering the risk of overfitting, we use external validation datasets from the same domain to evaluate the generalizability of \ours. We select two domain-specific subsets from PubMedVision~\citep{chen2024huatuogpt}, i.e., fundus digital photography and microscopy image, for ophthalmology and pathology, respectively. The results show that \ours\ still significantly outperforms other Med-LVLMs on the external validation datasets, indicating \ours\ performs well when generalized to external datasets, demonstrating its strong generalization capability.

\begin{table}[htbp]
\centering
\footnotesize
\caption{Performance comparison of models on external validation datasets.}
\label{tab:oph_path_performance}
\begin{tabular}{@{}lccccccc@{}}
\toprule
\multirow{2}{*}{Model} & \multicolumn{3}{c}{Ophthalmology} & \multicolumn{3}{c}{Pathology} \\ \cmidrule(lr){2-4} \cmidrule(lr){5-7}
                       & BLEU  & ROUGE-L & METEOR & Acc   & F1    & AUC   \\ \midrule
LLAVA-Med-1.5          & 17.11 & 20.05   & 17.09  & 59.65 & 71.90 & 54.87 \\
\ours\                & 22.64 & 14.98   & 17.85  & 62.88 & 72.24 & 59.69 \\ \bottomrule
\end{tabular}
\end{table}

\subsubsection{Detailed BLEU Score}
\label{sec:bleu}
We report the average BLEU score above. Detailed results are provided in Table~\ref{tab:bleu}.

\begin{table}[htbp]
    \centering
    \footnotesize
    \caption{BLEU Score (\%) of different methods based on LLaVA-Med-1.5 on report generation task.}
    \vspace{-1em}
    \resizebox{\linewidth}{!}{
    \begin{tabular}{l|cccccccc|cccc}
    \toprule
        \multirow{2}{*}{Models} & \multicolumn{8}{c|}{\textbf{Radiology}} & \multicolumn{4}{c}{\textbf{Ophthalmology}} \\ 
        \cmidrule(r){2-9} \cmidrule(r){10-13}
        & \multicolumn{4}{c}{IU-Xray} & \multicolumn{4}{c}{MIMIC-CXR} & \multicolumn{4}{c}{Harvard-FairVLMed}  \\ \cmidrule(r){2-5} \cmidrule(r){6-9} \cmidrule(r){10-13}
        & BLEU-1 & BLEU-2 & BLEU-3 & BLEU-4 & BLEU-1 & BLEU-2 & BLEU-3 & BLEU-4 & BLEU-1 & BLEU-2 & BLEU-3 & BLEU-4 \\
        \midrule
        LLaVA-Med-1.5 & 17.69 &	10.55 &	6.47 & 3.83 & 21.82 & 13.35 & 6.11 & 3.64 & 32.57 & 19.86 & 9.11 & 5.38 \\
        \midrule
        + Greedy & 21.04 & 12.57 & 5.75 & 3.35 & 29.91 & 18.26 & 8.27 & 5.03 & 32.40 & 19.82 & 9.04 & 5.37 \\
        + Beam Search & 21.78 & 12.71 & 6.05 & 3.63 & 30.55 & 17.79 & 8.49 & 5.09 & 33.07 & 19.14 & 9.14 & 5.48  \\
        + DoLa & 21.22 & 12.39 & 5.90 & 3.54 & 30.80 & 17.97 & 8.58 & 5.15 & 32.87 & 19.02 & 9.08 & 5.45 \\
        + OPERA & 19.79 & 11.19 & 5.33 & 3.20 & 27.72 & 16.05 & 7.65 & 4.59 & 29.90 & 17.45 & 8.32 & 4.99  \\ 
        + VCD & 19.35 & 10.94 & 5.21 & 3.13 & 27.27 & 15.76 & 7.51 & 4.51 & 30.14 & 17.61 & 8.39 & 5.04 \\ \midrule
        + MedDr & 22.27 & 12.99 & 6.19 & 3.71 & 33.43 & 19.33 & 9.22 & 5.53 & 35.64 & 20.61 & 9.82 & 5.89  \\
        + FactMM-RAG & 26.45 & 15.25 & 7.26 & 4.36 & 33.64 & 19.44 & 9.27 & 5.56 & 37.47 & 21.64 & 10.30 & 6.18 \\
        + RULE & 49.56 & 28.61 & 13.62 & 8.17 & 33.47 & 19.36 & 9.23 & 5.54 & 40.21 & 23.26 & 11.08 & 6.66 \\ 
        \midrule
        \ours\  & 56.48 & 32.67 & 15.56 & 9.34 & 41.81 & 24.18 & 11.52 & 6.92 & 44.65 & 25.79 & 12.29 & 7.38 \\
    \bottomrule
    \end{tabular}
    }
    \label{tab:bleu}
\end{table}

\subsubsection{Deeper Analysis of Retriever}
We have tried training a general retriever by mixing images from all modalities together, instead of using a domain-specific retriever. We conduct experiments based on BiomedCLIP and MedCLIP, but the results are unsatisfactory. Then we adopt an MoE (Mixture of Experts) architecture~\citep{shazeer2017outrageously,riquelme2021scaling}. Based on CLIP-MoE, we fine-tune CLIP-MoE~\citep{zhang2024clip} with mixing images from all medical imaging modalities, but the performance is still suboptimal. This might be because CLIP-MoE is not pretrained on large-scale biomedical data. All the results are reported in Table~\ref{tab:ret}. Considering model performance, we ultimately adopt a domain-specific retriever architecture. In fact, this approach is both flexible and scalable. Similar to a general retriever, encountering a completely new modality may still require retraining the retriever to achieve good retrieval performance, which incurs additional costs. For mixed datasets, as the number of modalities increases, training a general retriever becomes increasingly challenging, making it difficult to achieve reliable retrieval within a single domain. We address this by using a domain identification module to classify the input image by modality and select the corresponding retriever. In the future, a potential solution could involve pretraining a general retriever on large-scale biomedical data using a Mixture of Experts (MoE) architecture to explore whether it is possible to develop a general retriever. 

\begin{table}[htbp]
\centering
\caption{Performance comparison based on different retrievers.}
\label{tab:ret}
\footnotesize
\begin{tabular}{@{}lccc@{}}
\toprule
\multirow{2}{*}{Model}       & \multicolumn{3}{c}{IU-Xray} \\ \cmidrule(lr){2-4} 
                             & Acc   & F1    & AUC   \\ \midrule
LLaVA-Med                   & 75.47 & 64.04 & 67.46 \\
+ RAG (BiomedCLIP-FT)       & 79.09 & 65.87 & 69.52 \\
+ RAG (MedCLIP-FT)          & 75.13 & 63.88 & 67.16 \\
+ RAG (CLIP-MoE-FT)         & 72.13 & 62.72 & 65.11 \\
+ RAG (Ours)                & 84.82 & 68.85 & 77.54 \\ \bottomrule
\end{tabular}
\end{table}

\subsubsection{Comparison under Few-Shot Setting}
All our experiments are conducted under a zero-shot setting. We conduct experiments on LLaVA-Med-1.5 using the same few-shot strategy as in Med-Flamingo. The results show that compared to the zero-shot setting, the model's performance significantly decreases, even with RAG applied. Our analysis of this phenomenon reveals that, unlike Med-Flamingo, LLaVA-Med does not use interleaved multimodal data for pretraining. As a result, it lacks the capability for few-shot learning. This point has been mentioned in some discussion forums and GitHub issues. In addition, LLaVA-1.5's unsatisfactory performance on multi-image understanding benchmarks also supports this observation~\citep{wang2024mementos,meng2024mmiu}.

\begin{table}[htbp]
\centering
\caption{Performance comparison under zero-shot and few-shot settings.}
\label{tab:zero_few_shot}
\begin{tabular}{@{}lccc@{}}
\toprule
\multirow{2}{*}{Model}         & \multicolumn{3}{c}{IU-Xray} \\ \cmidrule(lr){2-4} 
                               & Acc   & F1    & AUC   \\ \midrule
LLaVA-Med (zero-shot)          & 75.47 & 64.04 & 67.46 \\
+MMed-RAG                      & 89.54 & 80.72 & 87.13 \\
LLAVA-Med (few-shot)           & 66.77 & 51.56 & 66.60 \\
+MMed-RAG                      & 84.10 & 71.92 & 86.40 \\ \bottomrule
\end{tabular}
\end{table}

\subsubsection{Performance Comparison of the Retriever}
Regarding the retriever's performance, as shown in Table~\ref{tab:recall_comparison}, we compared the performance of our retriever with several CLIP-based models on radiology datasets for image-to-text retrieval. The results demonstrate that our retriever significantly outperforms the other models in retrieval performance.

\begin{table}[htbp]
\centering
\caption{Performance comparison of different retrievers on Recall@1 (R@1) and Recall@5 (R@5) metrics.}
\label{tab:recall_comparison}
\footnotesize
\begin{tabular}{@{}lcc@{}}
\toprule
Model       & R@1  & R@5  \\ \midrule
CLIP        & 3.91 & 7.88 \\
PubMedCLIP  & 1.47 & 1.64 \\
MedCLIP     & 6.74 & 12.69 \\
BiomedCLIP  & 15.7 & 23.8 \\
PMC-CLIP    & 12.3 & 21.2 \\
Ours        & 45.6 & 71.8 \\ \bottomrule
\end{tabular}
\end{table}

\subsubsection{Rationale-Guided RAG}
For retrieved information, we minimize noise by optimizing the number of retrieved contexts \(k\) (e.g., Adaptive Retrieved Context Selection in Section 3.2). Following this, we introduce RAG-PT to specifically address the misalignment issues that arise after incorporating RAG, thereby strengthening Med-LVLM to balance its internal knowledge and external retrieval information. We employ a rationale-guided approach~\citep{wei2024instructrag} that uses LLM to explicitly learn denoising of retrieved content through self-synthesized rationales. First, given a question, the retrieved documents, and the ground truth from the training set, we prompt a powerful Med-LLM (i.e., LLaMA3-Med42-70B~\citep{christophe2024med42}) to generate a rationale. This rationale explains how to derive the answer from potentially noisy inputs. Next, we use the synthesized rationale from the previous step to guide another smaller Med-LLM (i.e., LLaMA3-Med42-7B~\citep{christophe2024med42}) to explicitly learn denoising of the retrieved documents through in-context learning and supervised learning. By employing this rationale-guided Med-LLM to filter noisy retrieval information, the reliability of our retrieved data improves. Experimental results show that after rationale-guided RAG, the model's performance further improved. 
\begin{table}[htbp]
\centering
\caption{Performance comparison on IU-Xray dataset, including RAG and Rationale-Guided RAG variants.}
\footnotesize
\label{tab:iu_xray_rag_comparison}
\begin{tabular}{@{}lccc@{}}
\toprule
\multirow{2}{*}{Model}         & \multicolumn{3}{c}{IU-Xray} \\ \cmidrule(lr){2-4} 
                               & Acc   & F1    & AUC   \\ \midrule
LLaVA-Med                     & 75.47 & 64.04 & 67.46 \\
+ RAG                         & 84.82 & 68.85 & 77.54 \\
\hspace{1em}+ RAG-PT          & 89.54 & 80.72 & 87.13 \\
+ Rationale-Guided RAG        & 85.38 & 69.23 & 77.90 \\
\hspace{1em}+ RAG-PT          & 89.91 & 80.86 & 87.32 \\ \bottomrule
\end{tabular}
\end{table}

\subsection{The Contribution of Domain-Specific Retrievers}
We design a domain-specific retriever leveraging a generalist Med-LVLM to retrieve information from a dedicated database based on the identified modality of the input medical image. Here, the domain identification models used are capable of reliably recognizing modalities with high accuracy (~99.83\% accuracy in our experiments). For radiology VQA tasks, input radiology images are classified as “radiology” by the model, enabling the retrieval of knowledge exclusively from the radiology database to enhance generation. All retrieved documents are specific to radiology and exclude other modalities.

\subsection{Explanation of Cross-Modality Alignment}
To construct preference pairs for cross-modality alignment, we first select a preferred response by having the model generate an answer using the correct medical image, clinical query, and retrieved knowledge, ensuring the response matches the ground-truth answer. Then, we select a dispreferred response by introducing an unrelated input image. This unrelated image is selected by finding the one with the lowest similarity to the target image and adding noise to distort it further. The dispreferred response is generated when the model uses this noisy, unrelated image along with the query and retrieved knowledge to still produce the correct answer. By comparing these pairs during training, the model learns to prioritize relevant and accurate inputs (e.g., the correct medical image) over noisy or irrelevant ones, improving cross-modality alignment.

\subsection{Analysis of Noisy Image in Cross-Modality Alignment}
In medical imaging, noise refers to random variations in image signals caused by hardware limitations or environmental factors~\citep{gravel2004method,sanchez2012medical}. However, the noise we refer to here pertains to images unrelated to the original image, generated through a two-step process: 1. We use a retriever to select images with the lowest similarity to the target image.  2. We introduce strong diffusion noise to these images.  As a result, the noisy images in our case are almost entirely random noise and are not merely examples of domain shifts, such as changes in lighting conditions. Refer to the third section of Figure 1 for examples, and additional examples are included in the Figure~\ref{fig:noise} for reference.

The motivation behind our design is that replacing the original image with a highly noisy image while adding retrieved information corresponding to the original image reveals a significant issue of cross-modal misalignment in the Med-LVLM—namely, it ignores the image information and directly copies the retrieved contexts. To mitigate this issue, we construct such preference pairs to specifically strengthen the model's cross-modal alignment capability.

\begin{figure}[t]
    \centering
    \includegraphics[width=0.95\linewidth]{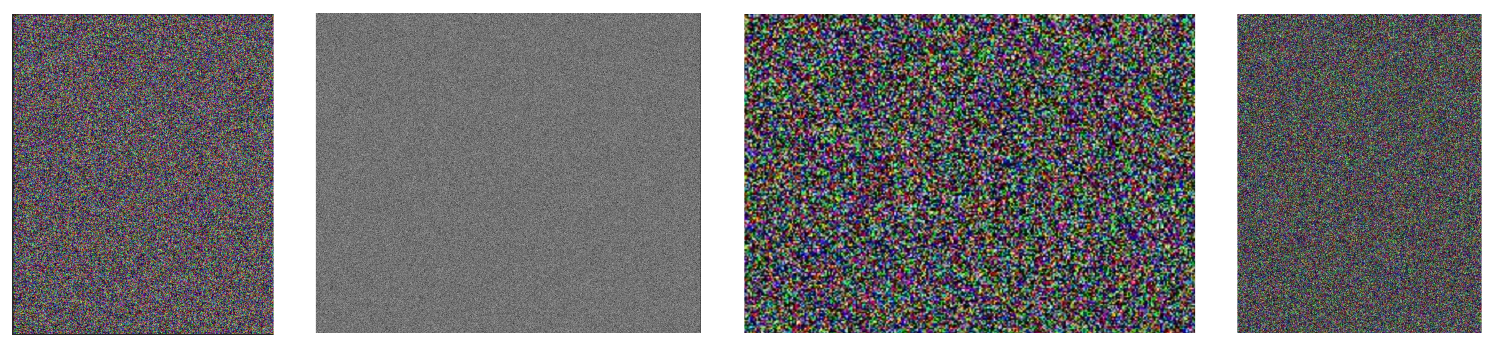}
    \vspace{-0.5em}
    \caption{Illustration Examples for noisy images in cross-modality alignment.
    }
    \label{fig:noise}
    \vspace{-1em}
\end{figure}

\subsection{Explanation of Over Reliance Rate}
The overall alignment issue arises from the conflict between retrieved information and the model's internal knowledge. For retrieved information, we cannot guarantee 100\% accuracy, so some noise is inevitable. The Over-Reliance (OR) rate shown in Figure 3 refers to the proportion of initially correct responses that become incorrect after adding the retrieved context, calculated relative to the total number of incorrect samples, not the total number of all samples. This rate represents the proportion of errors caused by over-reliance, rather than indicating poor performance of the retriever. Through RAG-PT, we can effectively mitigate this issue, significantly reducing the OR rate.

\newpage
\section{Proofs for Theoretical Results in Section \ref{thm anal}}
Here we provide proofs for the results in Section \ref{thm anal}.
\subsection{Notations}
Let $x_v, y, x_t,x_r$ be input medical image, ground-truth answer, question, and retrieved information, respectively. Denote $(x_w,y_{w,o}) \sim q_p(x_w,y_{w,o}|x_t,x_r)$ and $(x_l,y_{l,o}) \sim q_l(x_l,y_{l,o}|x_t,x_r)$ as distributions of the preferred responses and dispreferred responses. Let $x$ denote $(x_v,x_r,x_t)$. We aim to a fine-tune a generative model $\pi_\theta(y|x,x_t)$ through DPO loss~\citep{rafailov2023direct}:
\begin{equation} \label{app eq:dpo loss}
\argmin_{\pi_\theta} \mathbb{E}_{(x_w,x_l,y_{w,o},y_{l,o}) \sim \mathcal{D}} 
U\left( 
\alpha \log \frac{\pi_\theta(y_{w,o} | x)}{\pi_{o}(y_{w,o} | x)}
- \alpha \log \frac{\pi_\theta(y_{l,o} | x)}{\pi_{o}(y_{l,o} | x)}
 \right).
\end{equation}
where $U(t) = \log(1 + \exp(-t))$. Define the weight of $x_v$ with respect to $\log\pi_\theta(y|x)$ as
\begin{equation}
    \text{wt}(x_v,\pi_\theta) := \E_{y\sim\pi_\theta(\cdot|x)}\left[
        \frac{\partial}{\partial x_v} \log \pi_\theta(y|x) \right]^2
\end{equation}

\subsection{Assumptions}
\begin{assumption} \label{app asump: space} (Large parameter space) Assume that $\pi(x_v,y|x_t,x_r)$ lies in the optimization space $\{\pi_\theta,\theta\in\Theta\}$ such that $\pi(x_v,y|x_t,x_r) \propto \pi_o(x_v,y|x_t,x_r)\left(\frac{q_w(x_v,y|x_t,x_r)}{q_l(x_v,y|x_t,x_r)}\right)^{\frac{1}{\alpha}}$
\end{assumption}

\noindent Assumption \ref{app asump: space} requires that the parameter space sufficiently large to ensure that $\pi_\theta$ can achieve its global optimum, allowing us to represent the optimizer with a closed form.

\begin{assumption} \label{app asump: fun preoperty}
    Let $h(x,y)$, abbreviate as $h$, be
    {\small \begin{equation}
        h := \left[
        \sum_y\pi_o(y|x)\left(\frac{q_w(y|x)}{q_l(y|x)}\right)^{\frac{1}{\alpha}}\right]^{-1}\left(\frac{q_w(y|x)}{q_l(y|x)}\right)^{\frac{1}{\alpha}}
    \end{equation}}
    Assume that $\text{wt}(x_v,\pi_o) < c^2$, where
    {\small
    \begin{equation}
        c = \sqrt{\left\Vert \sqrt{\pi_o(y|x)}\cdot\frac{\partial}{\partial x_v}h \right\Vert_2^2 + \int \left(\frac{\partial}{\partial x_v}h\right)^2\frac{\pi_o(y|x)}{h} d y} - \left\Vert \sqrt{\pi_o(y|x)}\cdot\frac{\partial}{\partial x_v}h \right\Vert_2
    \end{equation}}
\end{assumption}

\begin{assumption} \label{app asump: fun preoperty 1}
    Let $h_1(x_v,x_t,x_r,y)$, abbreviate as $h_1$, be
    {\small
    \begin{equation}
    h_1:=
    \left[\sum_y\pi_o(y|x)\left(\frac{q_w^1(y|x_v,x_t,x_r)+q_w^2(y|x_v,x_t)}{q_l^1(y|x_v,x_t) + q_l^2(y|x_v,x_t,x_r)}\right)^{\frac{1}{\alpha}}\right]^{-1}\left(\frac{q_w^1(y|x_v,x_t,x_r)+q_w^2(y|x_v,x_t)}{q_l^1(y|x_v,x_t) + q_l^2(y|x_v,x_t,x_r)}\right)^{\frac{1}{\alpha}}
    \end{equation}}
    Assume that $\text{wt}(x_r,\pi_o) < c_1^2$ and $\text{wt}(\Tilde{x}_r,\pi_o) > c_2^2$, where
    {\small\begin{equation}
    \begin{aligned}
        & c_1 = \sqrt{\left\Vert \sqrt{\pi_o}\cdot\frac{\partial h_1}{\partial x_r} \right\Vert_2^2 + \int \left(\frac{\partial h_1}{\partial x_r}\right)^2\frac{\pi_o}{h_1} d y} - \left\Vert \sqrt{\pi_o}\cdot\frac{\partial h_1}{\partial x_r}\right\Vert_2 \\
        & c_2 = \sqrt{\left\Vert \sqrt{\pi_o}\cdot\frac{\partial h_1}{\partial \Tilde{x}_r} \right\Vert_2^2 
        + \int \left( \frac{\partial h_1}{\partial \Tilde{x}_r}\right)^2\frac{\pi_o}{h_1} + \left(\frac{\partial \pi_o}{\partial \Tilde{x}_r}\right)^2\frac{h_1}{\pi_o}
        d y} + \left\Vert \sqrt{\pi_o}\cdot\frac{\partial h_1}{\partial \Tilde{x}_r} \right\Vert_2
    \end{aligned}       
    \end{equation}}
\end{assumption}

\subsection{Proofs}
\begin{lemma}
    \label{app lem: optimizer form}
    Suppose that Assumption \ref{app asump: space} hold, optimizing equation \ref{app eq:dpo loss} gives
    \begin{equation} \label{equ: opt model}
        \pi_\theta(y|x) \propto \pi_o(y|x)\left(\frac{q_w(y|x)}{q_l(y|x)}\right)^{\frac{1}{\alpha}}
    \end{equation}
\end{lemma} 
\noindent Lemma \ref{app lem: optimizer form} indicates that the model tends to increase $\pi_o(y|x)$ if $q_w(y|x)> q_l(y|x)$, which is more likely to occur when $(x_v,y)$ represents a preferred sample given $x_t$ and $x_r$. Below, we provide an application of Lemma \ref{app lem: optimizer form} using a linear regression example. Lemma \ref{app lem: optimizer form} is proved with Lemma \ref{app lem: loss function property} and Lemma \ref{lem: closed form soultion}.
\begin{lemma} \label{app lem: loss function property}
    (Lemma C.1 in \cite{chen2024selfplayfinetuningconvertsweak}) For $a,b>0$, the following inequality holds
    $$a \cdot U(t) + b \cdot U(-t) \geq a \log(1 + b/a) + b\log(1+a/b)$$
    and equality holds if and only if $t = \log(a/b)$
\end{lemma}
\begin{lemma} \label{lem: closed form soultion}
Denote 
$$ \left\{
    \begin{array}{ll}
    p_1(x_w,y_{w,o},x_l,y_{l,o}|x_t,x_r) & = q_w(x_w,y_{w,o}|x_t,x_r) \cdot q_l(x_l,y_{l,o}|x_t,x_r)\\
    p_2(x_w,y_{w,o},x_l,y_{l,o}|x_t,x_r) & = q_l(x_w,y_{w,o}|x_t,x_r) \cdot q_w(x_l,y_{l,o}|x_t,x_r)
    \end{array}
    \right.$$
and abbreviated as $p_1$ and $p_2$ for notational convenience. Then,
\begin{equation}
    \begin{aligned}
       &  2\E_{\gD}\left[U\left(f(x_w,y_{w,o},x_t,x_r) - f(x_l,y_{l,o},x_t,x_r)\right)\right] \\
        \geq & 2\log2 - \KL\left(p_1 \big{\Vert} \frac{p_1 + p_2}{2}\right) - \KL\left(p_2 \big{\Vert} \frac{p_1 + p_2}{2}\right)
    \end{aligned}
\end{equation}
Equality holds if and only if
\begin{equation} \label{equ: closed form}
        f(x,y) = g(x) + \log\frac{q_w(x_v,y|x_t,x_r)}{q_l(x_v,y|x_t,x_r)}
    \end{equation}
where $g(x)$ is any function that is possibly dependent on $x_v$, $x_t$ and $x_r$.
\begin{proof}
    \begin{equation}
        \begin{aligned}
            & 2\E_{\gD}\left[U\left(f(x_w,y_{w,o},x_t,x_r) - f(x_l,y_{l,o},x_t,x_r)\right)\right]\\
            = & \int q(x_t,x_r) \cdot p_1 \cdot U\left(f(x_w,y_{w,o},x_t,x_r) - f(x_l,y_{l,o},x_t,x_r)\right) d x d y\\
            & + \int q(x_t,x_r) \cdot p_2 \cdot U\left(f(x_l,y_{l,o},x_t,x_r) - f(x_w,y_{w,o},x_t,x_r)\right) d x d y\\ 
            \geq & \int q(x_t,x_r) \left[p_1 \cdot \log\left(1 + \frac{p_2}{p_1}\right) + p_2 \cdot \log\left(1 + \frac{p_1}{p_2}\right)\right]d x d y \\
            = & 2\log2 + \int q(x_t,x_r) \left[p_1 \cdot \log\left( \frac{p_1 + p_2}{2p_1}\right) + p_2 \cdot \log\left( \frac{p_1 + p_2}{2p_2}\right) \right]d x d y\\
            = & 2\log2 - KL\left(p_1 \big{\Vert} \frac{p_1 + p_2}{2}\right) - KL\left(p_2 \big{\Vert} \frac{p_1 + p_2}{2}\right)
        \end{aligned}
    \end{equation}
    where the first inequality follows from Lemma \ref{app lem: loss function property}. For equivalence,
    \begin{equation}
    \begin{aligned}
        f(x,y_{w,o},x_t,x_r) - f(x_l,y_{l,o},x_t,x_r) & = \log\frac{q_w(x_w,y_{w,o}|x_t,x_r) \cdot q_l(x_l,y_{l,o}|x_t,x_r)}{q_l(x_w,y_{w,o}|x_t,x_r) \cdot q_w(x_l,y_{l,o}|x_t,x_r)}
    \end{aligned}
    \end{equation}
    Thus, for any $x_w, y_{w,o},x_l,y_{l,o},x_t,x_r$, 
    \begin{equation} \label{equ:f form}
        f(x_w,y_{w,o},x_t,x_r) - \log\frac{q_w(x_w,y_{w,o}|x_t,x_r)}{q_l(x_w,y_{w,o}|x_t,x_r)} = f(x_l,y_{l,o},x_t,x_r) - \log\frac{q_w(x_l,y_{l,o}|x_t,x_r)}{q_l(x_l,y_{l,o}|x_t,x_r)}
    \end{equation}
    Therefore, equation \ref{equ:f form} holds if and only if there exists some $g(x_v,x_t,x_r)$ such that 
    \begin{equation}
        f(x_v,x_t,x_r,y) = g(x_t,x_r) + \log\frac{q_w(x_v,y|x_t,x_r)}{q_l(x_v,y|x_t,x_r)}
    \end{equation}
\end{proof} 
\end{lemma}
\noindent Lemma \ref{lem: closed form soultion} provides a closed-form solution to equation \ref{app eq:dpo loss} if the parameter space is sufficiently large. This lemma is crucial for the proof Lemma \ref{app lem: optimizer form}, which follows below
\begin{proof}
According to the Assumption \ref{app asump: space}, we have
        \begin{equation}
          \pi(x_v,y|x_t,x_r) = \hat{g}(x_t,x_r) \pi_o(x_v,y|x_t,x_r)\left(\frac{q_w(x_v,y|x_t,x_r)}{q_l(x_v,y|x_t,x_r)}\right)^{\frac{1}{\alpha}}
    \end{equation}
After reparameterization, 
\begin{equation}
    \alpha\log\left(\frac{\pi(x_v,y|x_t,x_r)}{\pi_o(x_v,y|x_t,x_r)}\right) = \alpha\log [\hat{g}(x_t,x_r)] + \log\frac{q_w(x_v,y|x_t,x_r)}{q_l(x_v,y|x_t,x_r)}
\end{equation}
which is the global minimum of 
\begin{equation}
    \argmin_f \E_{\gD}\left[U\left(f(x_w,y_{w,o},x_t,x_r) - f(x_l,y_{l,o},x_t,x_r)\right)\right]
\end{equation}
by Lemma \ref{lem: closed form soultion}. Since $\pi(x_v,y|x_t,x_r) \in \{\pi_\theta, \theta\in \Theta\}$ lies in the optimization space, we have
\begin{equation}
\begin{aligned} \label{equ: eqv loss}
    \begin{aligned}
        & \min_f \E_{\gD}U\left(f(x_w,y_{w,o},x_t,x_r) - f(x_l,y_{l,o},x_t,x_r)\right) \\
        = & \min_{\pi_{\theta}}\E_{\gD}
        U\left( \alpha \log \frac{\pi_\theta(y_{w,o} | x_w, x_t,x_r)}{\pi_{o}(y_{w,o} | x_w,x_t,x_r)} - \alpha \log \frac{\pi_\theta(y_{l,o} | x_l,x_t,x_r)}{\pi_{o}(y_{l,o} | x_l,x_t,x_r)}\right)
    \end{aligned}   
\end{aligned}
\end{equation}
and $\pi_\theta(x_v,y|x_t,x_r)$ is the optimizer of equation \ref{equ: eqv loss}, which gives
\begin{equation}
    \begin{aligned}
         & \alpha\log\left(\frac{\pi_\theta(x_v,y|x_t,x_r)}{\pi_o(x_v,y|x_t,x_r)}\right) = g(x_t,x_r) + \log\frac{q_w(x_v,y|x_t,x_r)}{q_l(x_v,y|x_t,x_r)} \\
        \Longrightarrow & \pi_\theta(x_v,y|x_t,x_r) =  \pi_o(x_v,y|x_t,x_r)\left(\frac{q_w(x_v,y|x_t,x_r)}{q_l(x_v,y|x_t,x_r)}\right)^{\frac{1}{\alpha}}\exp\left(\frac{1}{\alpha}g(x_t,x_r)\right)
    \end{aligned}
\end{equation}
Then
\begin{equation}
        \begin{aligned}
            \pi_\theta(y|x)  & = \frac{\pi_\theta(x_v,y|x_t,x_r)}{\pi_{\theta}(x|x_t,x_r)} = \frac{\pi_o(x_v,y|x_t,x_r)\left(\frac{q_w(x_v,y|x_t,x_r)}{q_l(x_v,y|x_t,x_r)}\right)^{\frac{1}{\alpha}}\exp\left(\frac{1}{\alpha}(g(x_t,x_r)\right)}{\sum_y\pi_o(x_v,y|x_t,x_r)\left(\frac{q_w(x_v,y|x_t,x_r)}{q_l(x_v,y|x_t,x_r)}\right)^{\frac{1}{\alpha}}\exp\left(\frac{1}{\alpha}(g(x_t,x_r)\right)} \\
            & = \frac{\pi_o(y|x)\left(\frac{q_w(x_v,y|x_t,x_r)}{q_l(x_v,y|x_t,x_r)}\right)^{\frac{1}{\alpha}}}{\sum_y\pi_o(y|x)\left(\frac{q_w(x_v,y|x_t,x_r)}{q_l(x_v,y|x_t,x_r)}\right)^{\frac{1}{\alpha}}} = \frac{\pi_o(y|x)\left(\frac{q_w(y|x_v,x_t,x_r)}{q_l(y|x_v,x_t,x_r)}\right)^{\frac{1}{\alpha}}}{\sum_y\pi_o(y|x)\left(\frac{q_w(y|x_v,x_t,x_r)}{q_l(y|x_v,x_t,x_r)}\right)^{\frac{1}{\alpha}}} 
        \end{aligned}
    \end{equation}
\end{proof}

\begin{corollary} \label{app cor:lin reg}
    Suppose that preferred responses $(x_w,y_w)$ and dispreferred responses $(x_l,y_l)$ satisfy $y_w = \beta x_w + \epsilon_1 $ and $y_l = \Tilde{\beta} x_l  + \epsilon_2$ respectively. DPO for $y = \theta x_v + \epsilon_3$ is based on reference model $y = \theta_{o} x_v + \epsilon_4$, where $\epsilon_i$'s are independent and follow standard normal distribution. Then,
    \begin{equation}
        \theta = \theta_{o} + \frac{1}{\alpha}(\beta - \Tilde{\beta})
    \end{equation}
\end{corollary}

\noindent Corollary \ref{app cor:lin reg} is a direct application of Lemma \ref{app lem: optimizer form}, indicating that the model updates coefficient $\theta_{o}$ towards the direction of $\beta$ for preferred responses and away from $\Tilde{\beta}$ for dispreferred responses. 

\begin{proof}
    Let $\phi(\cdot)$ denote the probability density function of standard normal, by Lemma \ref{app lem: optimizer form},
    \begin{equation} 
    \begin{aligned}
         & \phi(y-\theta x)
          \propto \phi(y-\theta_{o}x)
         \left(\frac{\phi(y-\beta x)}{\phi(y-\Tilde{\beta} x)}\right)^{\frac{1}{\alpha}} \\
        \Longrightarrow &\exp\left(\frac{1}{2}y^2 - \theta_1 x y\right) 
         \propto \exp\left(\frac{1}{2}y^2 -\theta_{o}x y\right)\cdot \exp\left(-\frac{1}{\alpha}(\beta - \Tilde{\beta}) xy\right) \\
         \Longrightarrow &\exp\left(\theta_1 x y\right) 
         \propto \exp\left(\theta_{o}x y\right)\cdot \exp\left( \frac{1}{\alpha}(\beta - \Tilde{\beta}) xy\right) \\
         \Longrightarrow & \theta = \theta_{o} + \frac{1}{\alpha}(\beta - \Tilde{\beta})
    \end{aligned}
\end{equation}

\end{proof}

\begin{lemma} \label{app lem: linear}
    For linear model $y = \theta_1 x_v + \theta_2 x_t +\epsilon$ such that $\epsilon\sim N (0,1)$, $\text{wt}(x_v,\pi_\theta) = \theta_1^2$
\end{lemma}
\begin{proof}
Let $\phi(\cdot)$ denote the probability density function of standard normal,
\begin{equation}
    \begin{aligned}
        \text{wt}(x_v,\pi_\theta) & = \int  \left(-\frac{1}{2}\frac{\partial}{\partial x_v} \left(y-\theta_1 x_v - \theta_2 x_t\right)^2\right)^2\phi(y-\theta_1 x_v - \theta_2 x_t) d y \\
        & = \theta_1^2 \int \left(y-\theta_1 x_v - \theta_2 x_t\right)^2\phi(y-\theta_1 x_v - \theta_2 x_t) d y \\
        & = \theta_1^2 \int \left(\theta_1 x_v + \theta_2 x_t - y\right)\frac{d \phi(y-\theta_1 x_v - \theta_2 x_t)}{d y}d y \\
        & = \theta_1^2 \int \phi(y-\theta_1 x_v - \theta_2 x_t)d y = \theta_1^2
    \end{aligned}
\end{equation}   
\end{proof}

\begin{theorem} \label{app thm: increase weight}
Suppose that Assumption \ref{app asump: fun preoperty} holds, then cross-modality increase the weight of $x_v$.
\begin{equation}
    \text{wt}(x_v,\pi_\theta) > \text{wt}(x_v,\pi_o)
\end{equation}
\end{theorem}

\begin{proof}
    By Lemma \ref{app lem: optimizer form}, we have
\begin{equation}
    \pi_\theta(y|x) = \pi_o(y|x) \cdot h(x,y),~~\int \pi_o(y|x) \cdot h(x,y) d y = 1
\end{equation}
Abbreviate $h(x,y)$ and $\pi_o(y|x_v,x_t)$ as $h$ and $\pi_o$ respectively, we have
\begin{equation} \label{eq: weight increase}
    \begin{aligned}
        \text{wt}(x_v,\pi_\theta) - \text{wt}(x_v,\pi_o)& \geq \int  
        \left(\frac{\frac{\partial}{\partial x_v}\pi_o}{\pi_o } + \frac{\frac{\partial}{\partial x_v}h}{h}\right)^2 \pi_o   h~d y -\text{wt}(x_v,\pi_o)\\
        & \geq \int \left[\frac{\partial}{\partial x_v}h\right]^2 \frac{\pi_o}{h} d y -2\sqrt{\text{wt}(x_v,\pi_o)}\cdot\left\Vert\sqrt{\pi_o}\cdot \frac{\partial}{\partial x_v}h\right\Vert_2 - \text{wt}(x_v,\pi_o)
    \end{aligned}
\end{equation}
the second inequality follows from Cauchy–Schwarz inequality
\begin{equation}
    \begin{aligned}
        \int \frac{\partial}{\partial x_v}\pi_o \cdot \frac{\partial}{\partial x_v}h~d y =\int \frac{\partial}{\partial x_v}\pi_o \cdot \frac{\sqrt{\pi_o}}{\sqrt{\pi_o}}\cdot\frac{\partial}{\partial x_v}h~d y  \leq \sqrt{\text{wt}(x_v,\pi_o)}\cdot\left\Vert \sqrt{\pi_o}\cdot\frac{\partial}{\partial x_v}h \right\Vert_2
    \end{aligned}
\end{equation}
Denote $c$ as
\begin{equation}
    c:=\sqrt{\left\Vert \sqrt{\pi_o}\cdot\frac{\partial}{\partial x_v}h \right\Vert_2^2 + \int \left(\frac{\partial}{\partial x_v}h\right)^2\frac{\pi_o}{h} d y} - \left\Vert \sqrt{\pi_o}\cdot\frac{\partial}{\partial x_v}h \right\Vert_2
\end{equation}
the last term in equation \ref{eq: weight increase} is equivalent to
\begin{equation}
    \left(c - \sqrt{\text{wt}(x_v,\pi_o)}\right) \cdot \left(\sqrt{\text{wt}(x_v,\pi_o)} + c + 2\left\Vert \sqrt{\pi_o}\cdot\frac{\partial}{\partial x_v}h \right\Vert_2\right)
\end{equation}
Thus, $\text{wt}(x_v,\pi_\theta) > \text{wt}(x_v,\pi_o)$ if $\sqrt{\text{wt}(x_v,\pi_o)} < c$.
\end{proof}

\begin{theorem} \label{app thm: increase weight x_r}
Suppose that Assumption \ref{app asump: fun preoperty 1} holds, the overall loss increase the weight of $x_r$ and decrease the weight of $\Tilde{x}_r$.
\begin{equation}
    \text{wt}(x_r,\pi_\theta) > \text{wt}(x_r,\pi_o),~~~\text{wt}(\Tilde{x}_r,\pi_\theta) < \text{wt}(\Tilde{x}_r,\pi_o)
\end{equation}
\end{theorem}

\begin{proof}
    The distribution of preferred responses can be considered as a mixture distribution: $q_w^1(x_v,y_{w,o}|x_t,x_r) + q_w^2(x_v,y_{w,o}|x_t)$. Similarly, for dispreferred responses, the distribution is represented as $q_l^1(x_v,y_{l,o}|x_t) + q_l^2(x_v,y_{l,o}|x_t,x_r)$. By Lemma \ref{app lem: optimizer form}, 
\begin{equation}
    \pi_\theta(y|x) = \pi_o(y|x) \cdot h_1(x,y),~~\int \pi_o(y|x) \cdot h_1(x,y) d y = 1
\end{equation}
Abbreviate $h_1(x,y)$ as $h_1$. Follow the same procedure in the proof of Theorem \ref{app thm: increase weight},
\begin{equation} \label{eq: weight increase x_r}
    \begin{aligned}
        \text{wt}(x_r,\pi_\theta) - \text{wt}(x_r,\pi_o)&
        \geq \int \left[\frac{\partial}{\partial x_r}h_1\right]^2 \frac{\pi_o}{h_1} d y -2\sqrt{\text{wt}(x_r,\pi_o)}\cdot\left\Vert\sqrt{\pi_o}\cdot \frac{\partial}{\partial x_r}h_1\right\Vert_2 - \text{wt}(x_r,\pi_o)\\
        & = \left(c_1 - \sqrt{\text{wt}(x_r,\pi_o)}\right) \cdot \left(\sqrt{\text{wt}(x_r,\pi_o)} + c_1 + 2\left\Vert \sqrt{\pi_o}\cdot\frac{\partial}{\partial x_r}h_1 \right\Vert_2\right)
    \end{aligned}
\end{equation}
where we apply Cauchy–Schwarz inequality in equation \ref{eq: weight increase x_r}.
\begin{equation}
     c_1 = \sqrt{\left\Vert \sqrt{\pi_o(y|x)}\cdot\frac{\partial}{\partial x_r}h_1 \right\Vert_2^2 + \int \left(\frac{\partial}{\partial x_r}h_1\right)^2\frac{\pi_o(y|x)}{h_1} d y} - \left\Vert \sqrt{\pi_o(y|x)}\cdot\frac{\partial}{\partial x_r}h_1\right\Vert_2
\end{equation}
Thus, $\text{wt}(x_r,\pi_\theta) > \text{wt}(x_r,\pi_o)$ if $\sqrt{\text{wt}(x_r,\pi_o)} < c_1$. Again, by Cauchy–Schwarz inequality
\begin{equation} \label{eq: weight decrease x_r}
    \begin{aligned}
        & \text{wt}(\Tilde{x}_r,\pi_\theta) - \text{wt}(\Tilde{x}_r,\pi_o)\\
        \leq & \int \left(\frac{\partial h_1}{\partial \Tilde{x}_r}\right)^2 \frac{\pi_o}{h_1}  + \left(\frac{\partial\pi_o}{\partial \Tilde{x}_r}\right)^2\frac{h_1}{\pi_o}d y +2\sqrt{\text{wt}(\Tilde{x_r},\pi_o)}\cdot\left\Vert\sqrt{\pi_o}\cdot \frac{\partial h_1}{\partial \Tilde{x}_r}\right\Vert_2 - \text{wt}(\Tilde{x}_r,\pi_o)\\
        = & - \left(\sqrt{\text{wt}(\Tilde{x_r},\pi_o)} - c_2\right) \cdot \left(\sqrt{\text{wt}(\Tilde{x_r},\pi_o)} - c_2 + 2\left\Vert \sqrt{\pi_o}\cdot\frac{\partial}{\partial \Tilde{x}_r}h_1 \right\Vert_2\right)
    \end{aligned}
\end{equation}
where
{\small
\begin{equation}
    c_2 = \sqrt{\left\Vert \sqrt{\pi_o}\cdot\frac{\partial}{\partial \Tilde{x}_r}h_1 \right\Vert_2^2 
        + \int 
        \left(\frac{\partial}{\partial \Tilde{x}_r}h_1\right)^2\frac{\pi_o}{h_1} + \left(\frac{\partial}{\partial \Tilde{x}_r}\pi_o\right)^2\frac{h_1}{\pi_o}
        d y} + \left\Vert \sqrt{\pi_o}\cdot\frac{\partial}{\partial \Tilde{x}_r}h_1 \right\Vert_2
\end{equation}}
Thus, $\text{wt}(x_r,\pi_\theta) < \text{wt}(x_r,\pi_o)$ if $\sqrt{\text{wt}(x_r,\pi_o)} > c_2$.
\end{proof}

\end{document}